\documentclass{article} 
\usepackage{iclr2023_conference,times}

\usepackage{amsmath,amsfonts,bm}

\def\eqref#1{equation~\ref{#1}}

\def\1{\bm{1}}

\DeclareMathAlphabet{\mathsfit}{\encodingdefault}{\sfdefault}{m}{sl}
\SetMathAlphabet{\mathsfit}{bold}{\encodingdefault}{\sfdefault}{bx}{n}

\newcommand{\E}{\mathbb{E}}

\newcommand{\R}{\mathbb{R}}

\DeclareMathOperator*{\argmax}{arg\,max}

\usepackage{hyperref}
\usepackage{url}

\usepackage{mathtools}
\usepackage{amsthm, amssymb}
\usepackage{bbm}
\usepackage{bm}
\usepackage{booktabs}       
\usepackage{amsfonts}       
\usepackage{nicefrac}       
\usepackage{microtype}      
\usepackage{xcolor}         

\def\R{{\mathbb{R}}}
\def\E{{\mathbb{E}}}

\def\L{{\mathcal{L}}}

\def\H{{\mathcal{H}}}
\def\S{{\mathcal{S}}}
\def\D{{\mathcal{D}}}
\def\O{{\mathcal{O}}}

\def\GE{{\mathcal{G}}}
\def\EE{{\mathcal{E}}}
\def\HV{{\boldsymbol{\mathcal{H}}}}
\newcommand{\norm}[1]{\left\|{#1}\right\|}
\newcommand{\Ra}[1]{\mathfrak{R}_m \left( {#1} \right)}
\newcommand{\FH}[2]{{#1} \circ {#2} }
\newcommand{\sh}[2]{ s_{\boldsymbol \epsilon}({#1}, {#2}) }

\newtheorem{theorem}{Theorem}

\newtheorem{proposition}{Proposition}
\newtheorem{lemma}{Lemma}

\newtheorem{assumption}{Assumption}
\newtheorem{defn}{Definition}

\title{Generalization and Estimation Error Bounds for Model-based Neural Networks}

\author{Avner Shultzman$^{1}$, Eyar Azar$^{1}$, Miguel R. D. Rodrigues$^2$ \& Yonina C. Eldar$^1$ \thanks{ This project has received funding from the European Research Council (ERC) under the European Union’s Horizon 2020 research, the innovation programme (grant agreement No. 101000967), and the Israel Science Foundation under Grant 536/22.
Y. C. Eldar and M. R. D. Rodrigues are supported by The Weizmann-UK Making Connections Programme (Ref. 129589). M. R. D. Rodrigues is also supported by the Alan Turing Institute.
The authors wish to thank Dr. Gholamali Aminian from the Alan Turing Institute, UK, for his contribution to the proofs' correctness.} \\
Faculty of Mathematics and Computer Science, Weizmann Institute of Science, Israel$^1$ \\
Faculty of Engineering Science, University College London, UK$^2$\\
\texttt{\{avner.shultzman,eyar.azar,yonina.eldar\}@weizmann.ac.il}\\
\texttt{m.rodrigues@ucl.ac.uk}
}

\iclrfinalcopy 
\begin{document}

\maketitle

\begin{abstract}
  Model-based neural networks provide unparalleled performance for various tasks, such as sparse coding and compressed sensing problems.
  Due to the strong connection with the sensing model, these networks are interpretable and inherit prior structure of the problem.
  In practice, model-based neural networks exhibit higher generalization capability compared to ReLU neural networks.
  However, this phenomenon was not addressed theoretically.
  Here, we leverage complexity measures including the global and local Rademacher complexities, in order to provide upper bounds on the generalization and estimation errors of model-based networks.
  We show that the generalization abilities of model-based networks for sparse recovery outperform those of regular ReLU networks, and derive practical design rules that allow to construct model-based networks with guaranteed high generalization.
  We demonstrate through a series of experiments that our theoretical insights shed light on a few behaviours experienced in practice, including the fact that ISTA and ADMM networks exhibit higher generalization abilities (especially for small number of training samples), compared to ReLU networks.
\end{abstract}

\section{Introduction}

Model-based neural networks provide unprecedented performance gains for solving sparse coding problems, such as the learned iterative shrinkage and thresholding algorithm (ISTA) \citep{GregorLeCun2010LISTA} and learned alternating direction method of multipliers (ADMM) \citep{Boyd2011LADMM}.
In practice, these approaches outperform feed-forward neural networks with ReLU nonlinearities. 

These neural networks are usually obtained from algorithm unrolling (or unfolding) techniques, which were first proposed by Gregor and LeCun \citep{GregorLeCun2010LISTA}, to connect iterative algorithms to neural network architectures.
The trained networks can potentially shed light on the problem being solved.
For ISTA networks, each layer represents an iteration of a gradient-descent procedure.
As a result, the output of each layer is a valid reconstruction of the target vector, and we expect the reconstructions to improve with the network's depth.
These networks capture original problem structure, which translates in practice to a lower number of required training data \citep{Monga2021Unrolling}.
Moreover, the generalization abilities of model-based networks tend to improve over regular feed-forward neural networks \citep{Behboodi2020MB_RC,Schnoor2021MB_RC}.

Understanding the generalization of deep learning algorithms has become an important open question. The generalization error of machine learning models measures the ability of a class of estimators to generalize from training to unseen samples, and avoid overfitting the training \citep{Jakubovitz2019GEinDL}.
Surprisingly, various deep neural networks exhibit high generalization abilities, even for increasing networks' complexities \citep{Neyshabur2015c,Belkin2019BiasVariance}.
Classical machine learning measures such as the Vapnik-Chervonenkis (VC) dimension \citep{Vapnik1991VCdim} and Rademacher complexity (RC) \citep{Bartlett2002Rademacher}, predict an increasing generalization error (GE) with the increase of the models' complexity, and fail to explain the improved generalization observed in experiments.
More advanced measures consider the training process and result in tighter bounds on the estimation error (EE), were proposed to investigate this gap, such as the local Rademacher complexity (LRC) \citep{Bartlett2005LRC}.
To date, the EE of model based networks using these complexity measures has not been investigated to the best of our knowledge.
\subsection{Our Contributions} 
In this work, we leverage existing complexity measures such as the RC and LRC, in order to bound the generalization and estimation errors of learned ISTA and learned ADMM networks.
\begin{itemize}
    \item We provide new bounds on the GE of ISTA and ADMM networks, showing that the GE of model-based networks is lower than that of the common ReLU networks.
    The derivation of the theoretical guarantees combines existing proof techniques for computing the generalization error of multilayer networks with new methodology for bounding the RC of the soft-thresholding operator, that allows a better understanding of the generalization ability of model based networks.
    
    \item The obtained bounds translate to practical design rules for model-based networks which guarantee high generalization.
    In particular, we show that a nonincreasing GE as a function of the network's depth is achievable, by limiting the weights' norm in the network.
    This improves over existing bounds, which exhibit a logarithmic increase of the GE with depth \citep{Schnoor2021MB_RC}.
    The GE bounds of the model-based networks suggest that under similar restrictions, learned ISTA networks generalize better than learned ADMM networks.
    
    \item We also exploit the LRC machinery to derive bounds on the EE of feed-forward networks, such as ReLU, ISTA, and ADMM networks.    The EE bounds depend on the data distribution and training loss.
    We show that the model-based networks achieve lower EE bounds compared to ReLU networks.
    
    \item 
    We focus on the differences between ISTA and ReLU networks, in term of performance and generalization.
    This is done through a series of experiments for sparse vector recovery problems.
    The experiments indicate that the generalization abilities of ISTA networks are controlled by the soft-threshold value.
    For a proper choice of parameters, ISTA achieves lower EE along with more accurate recovery.
    The dependency of the EE as a function of $\lambda$ and the number of training samples can be explained by the derived EE bounds.

\end{itemize}

\subsection{Related Work}
Understanding the GE and EE of general deep learning algorithms is an active area of research.
A few approaches were proposed, which include considering networks of weights matrices with bounded norms (including spectral and $L_{2,1}$ norms) \citep{Bartlett2017Spectral,Sokolic2017NormBoundMiguel}, and analyzing the effect of multiple regularizations employed in deep learning, such as weight decay, early stopping, or  drop-outs, on the generalization abilities \citep{Neyshabur2015Regularization,Gao2016Dropout,Amjad2021Miguel}.
Additional works consider global properties of the networks, such as a bound on the product of all Frobenius norms of the weight matrices in the network \citep{Shamir2018SizeIndep}.
However, these available bounds do not capture the GE behaviour as a function of network depth, where an increase in depth typically results in improved generalization.
This also applies to the bounds on the GE of ReLU networks, detailed in Section \ref{sec:GE of ReLU}.

Recently, a few works focused on bounding the GE specifically for deep iterative recovery algorithms \citep{Behboodi2020MB_RC,Schnoor2021MB_RC}.
They focus on a broad class of unfolded networks for sparse recovery, and provide bounds which scale logarithmically with the number of layers \citep{Schnoor2021MB_RC}. 
However, these bounds still do not capture the behaviours experienced in practice.

Much work has also focused on incorporating the networks' training process into the bounds. 
The LRC framework due to Bartlett, Bousquet, and Mendelson \citep{Bartlett2005LRC} assumes that the training process results in a smaller class of estimation functions, such that the distance between the estimator in the class and the empirical risk minimizer (ERM) is bounded.
An additional related framework is the effective dimensionality due to Zhang \citep{Zhang2002EffectiveDimension}.
These frameworks result in different bounds, which relate the EE to the distance between the estimators.
These local complexity measures were not applied to model-based neural networks.

Throughout the paper we use boldface lowercase and uppercase letters to denote vectors and matrices respectively.
The $L_1$ and $L_2$ norms of a vector $\boldsymbol x$ are written as $\norm{\boldsymbol x}_1$ and $\norm{\boldsymbol x}_2$ respectively, and the $L_{\infty}$ (which corresponds to the maximal $L_1$ norm over the matrix's rows) and spectral norms of a matrix $\boldsymbol X$, are denoted by $\norm{\boldsymbol X}_{\infty}$ and $\norm{\boldsymbol X}_{\sigma}$ respectively.
We denote the transpose operation by $(\cdot)^T$.
For any function $f$ and class of functions $\H$, we define $
    \FH{f} \H= \left\{ x \mapsto f \circ h(x): h \in \H \right\}.
$

\section{Preliminaries}
\label{sec:Preliminaries}


\subsection{Network architecture}
\label{sec:Network_Architecture}

We focus on model-based networks for sparse vector recovery, applicable to the linear inverse problem
\begin{equation}
\begin{aligned}
    \boldsymbol y = \boldsymbol A \boldsymbol x + \boldsymbol e
    \label{eq:linear_inverse_problem}
\end{aligned}
\end{equation}
where $\boldsymbol y \in \R^{n_y}$ is the observation vector with $n_y$ entries, $\boldsymbol x \in \R^{n_x}$ is the target vector with $n_x$ entries, with $n_x > n_y$, $\boldsymbol A \in \R^{n_y \times n_x}$ is the linear operator, and $\boldsymbol e \in \R^{n_y}$ is  additive noise.
The target vectors are sparse with sparsity rate $\rho$, such that at most $\lfloor \rho n_x \rfloor$ entries are nonzero.
The inverse problem consists of recovering the target vector $\boldsymbol x$, from the observation vector $\boldsymbol y$.

Given that the target vector is assumed to be sparse, recovering $\boldsymbol x$ from $\boldsymbol y$ in (\ref{eq:linear_inverse_problem}) can be formulated as an optimization problem,
such as least absolute shrinkage and selection operator (LASSO) \citep{Tibshirani2013LASSO},
that can be solved with well-known iterative methods including ISTA and ADMM. 
To address more complex problems, such as an unknown linear mapping $\boldsymbol A$, and to avoid having to fine tune parameters, these algorithms can be mapped into model-based neural networks using unfolding or unrolling techniques \citep{GregorLeCun2010LISTA,Boyd2011LADMM,Monga2021Unrolling,Yang2018ADMMCSNet}.
The network's architecture imitates the original iterative method's functionality and enables to learn the models' parameters with respect to a set of training examples.

We consider neural networks with $L$ layers (referred to as the network's depth), which corresponds to the number of iterations in the original iterative algorithm.
The layer outputs of an unfolded ISTA network $\boldsymbol h_I^{l},\ l \in [1, L]$, are defined by the following recurrence relation, shown in Fig. \ref{fig:Model-based networks architectures}:
\begin{equation}
\begin{aligned}
    \boldsymbol h_I^{l} = \S_{\lambda} \left( \boldsymbol W^{l } \boldsymbol h_I^{l - 1} + \boldsymbol b\right), \ \ \ \boldsymbol h_I^0 = \S_{\lambda} (\boldsymbol b)
    \label{eq:ISTA_unfolded}
\end{aligned}
\end{equation}
where $\boldsymbol W^{l} \in \R^{n_x \times n_x},\ l \in [1,L]$ are the weights matrices corresponding to each of the layers, with bounded norms $||\boldsymbol W^{l}||_{\infty} \leq B_l$.
We further assume that the $L_2$ norm of $\boldsymbol W^{1}$ is bounded by $B_1$.
The vector $\boldsymbol b = \boldsymbol A^T \boldsymbol y$ is a constant bias term that depends on the observation $\boldsymbol y$, where we assume that the initial values are bounded, such that $||\boldsymbol h_I^{0}||_1 \leq B_0$.
In addition, $\S_{\lambda}(\cdot)$ is the elementwise soft-thresholding operator
\begin{equation}
\begin{aligned}
    \S_{\lambda}(\boldsymbol h) = \text{sign}(\boldsymbol h) \max(|\boldsymbol h| - \lambda, \boldsymbol 0)
    \label{eq:soft-thresholding_operator}
\end{aligned}
\end{equation}
where the functions $\text{sign}(\cdot)$ and $\max(\cdot)$ are applied elementwise, and $\boldsymbol 0$ is a vector of zeros.
As $\S_{\lambda}(\cdot)$  is an elementwise function it preserves the input's dimension, and can be applied on scalar or vector inputs.
The network's prediction is given by the last layer in the network $\boldsymbol{ \hat{x}} = \boldsymbol h_I^L (\boldsymbol y)$.
We note that the estimators are functions mapping $\boldsymbol y$ to $\hat{\boldsymbol{x}}$, $\boldsymbol h_I^L: \mathbb{R}^{n_y} \xrightarrow[]{} \mathbb{R}^{n_x}$, characterized by the weights, i.e. $\boldsymbol h_I^L = \boldsymbol h_I^L ( \{ \boldsymbol W^l \}_{l = 1}^{L } )$.
The class of functions representing the output at depth $L$ in an ISTA network, is $\HV_I^L = \{\boldsymbol h_I^L ( \{ \boldsymbol W^l \}_{l = 1}^{L} ): ||\boldsymbol W^l||_{\infty} \leq B_l\ l \in [1, L ],\ ||\boldsymbol W^1||_{2} \leq B_1 \}$.

Similarly, the $l$th layer of unfolded ADMM is defined by the following recurrence relation
\begin{equation}
\begin{aligned}
    \boldsymbol h_A^l & = \boldsymbol W^l \left( \boldsymbol z^{l - 1} + \boldsymbol u^{l - 1} \right) + \boldsymbol b \\
    \boldsymbol z^l & = \S_{\lambda} \left( \boldsymbol h_A^{l} - \boldsymbol u^{l - 1} \right), & \boldsymbol z^0 = \boldsymbol 0  \\
    \boldsymbol u^l & = \boldsymbol u^{l - 1} - \gamma \left( \boldsymbol h_A^{l} - \boldsymbol z^{l} \right), & \boldsymbol u^0 = \boldsymbol 0 
    \label{eq:ADMM_unfolded}
\end{aligned}
\end{equation}
where $\boldsymbol 0$ is a vector of zeros, $\boldsymbol b = \boldsymbol A^T \boldsymbol y$ is a constant bias term, and $\gamma > 0$ is the step size derived by the original ADMM algorithm, as shown in Fig. \ref{fig:Model-based networks architectures}.
The estimators satisfy $\boldsymbol h_A^L: \mathbb{R}^{n_y} \xrightarrow[]{} \mathbb{R}^{n_x}$, and the class of functions representing the output at depth $L$ in an ADMM network, is $\HV_A^L = \{\boldsymbol h_A^L ( \{ \boldsymbol W^l \}_{l = 1}^{L} ): ||\boldsymbol W^l||_{\infty} \leq B_l\ l \in [1, L ],\ ||\boldsymbol W^1||_{2} \leq B_1 \}$, where we impose the same assumptions on the weights matrices as ISTA networks.

For a depth-$L$ ISTA or ADMM network, the learnable parameters are the weight matrices $\{ \boldsymbol W^{l} \}_{l=1}^L$.
The weights are learnt by minimizing a loss function $\L$ on a set of $m$ training examples $S = \left\{ \left( \boldsymbol x_i, \boldsymbol y_i \right)\right\}_{i = 1}^m$, drawn from an unknown distribution $\D$, consistent with the model in (\ref{eq:linear_inverse_problem}).
We consider the case where the per-example loss function is obtained by averaging over the example per-coordinate losses:
\begin{equation}
\begin{aligned}
    \L\left( \boldsymbol h(\boldsymbol y), \boldsymbol x\right) = \frac{1}{n_x} \sum_{j = 1}^{n_x} \ell \left( \boldsymbol h_j(\boldsymbol y), \boldsymbol x_j \right)
    \label{eq:loss decomposition}
\end{aligned}
\end{equation}
where $h_j(\boldsymbol y)$ and $\boldsymbol x_j$ denote the $j$th coordinate of the estimated and true targets, and $\ell: \mathbb{R} \times \mathbb{R} \xrightarrow[]{} \mathbb{R_+}$ is $1$-Lipschitz in its first argument.
This requirement is satisfied in many practical settings, for example with the $p$-power of $L_p$ norms, and is also required in a related work \citep{XuC2016}.
The loss of an estimator $\boldsymbol h \in \HV$ which measures the difference between the true value $\boldsymbol x$ and the estimation $\boldsymbol h(\boldsymbol y)$, is denoted for convenience by $\L(\boldsymbol h) = \L\left( \boldsymbol h(\boldsymbol y), \boldsymbol x\right)$.

There exists additional forms of learned ISTA and ADMM networks, which include learning an additional set of weight matrices affecting the bias terms \citep{Monga2021Unrolling}.
Also, the optimal value of $\lambda$ generally depends on the target vector sparsity level.
Note however that for learned networks, the value of $\lambda$ at each layer can also be learned.
However, here we focus on a more basic architecture with fixed $\lambda$ in order to draw theoretical conclusions.

\begin{figure}[!h]
        \centering
        \includegraphics[width = 2.6in]{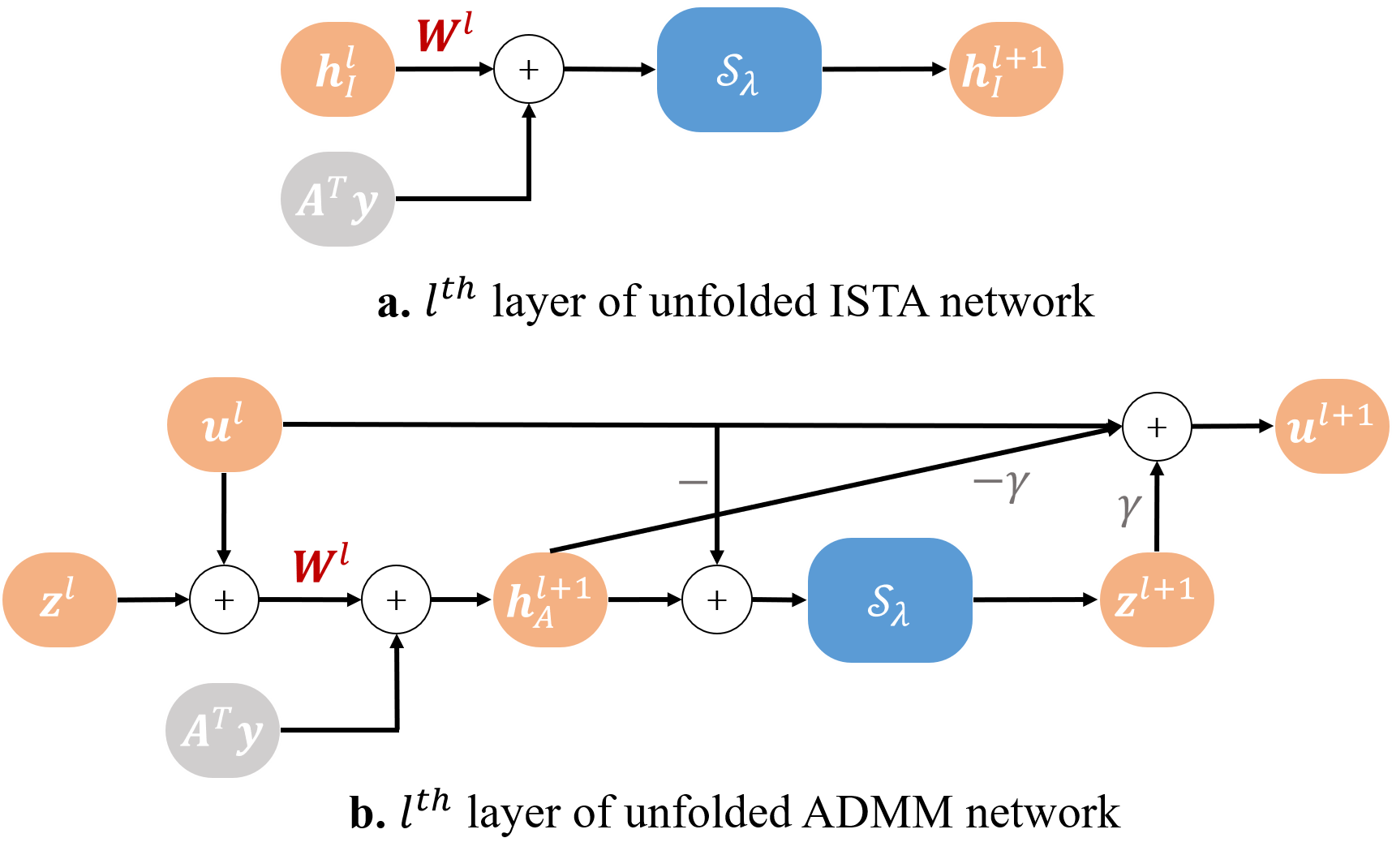}
        \caption{A single layer in the unfolded networks. \textbf{a.} Unfolded ISTA. \textbf{b.} Unfolded ADMM. The learnable parameters are the weight matrices (marked in red).}
    \label{fig:Model-based networks architectures}
\end{figure}

\subsection{Generalization and Estimation Errors}
\label{sec:Generalization_and_Estimation_Errors}

In this work, we focus on upper bounding the GE and EE of the model-based neural networks of Fig. \ref{fig:Model-based networks architectures}.
The GE of a class of estimation functions $\boldsymbol h \in \HV$, such that $\boldsymbol h: \mathbb{R}^{n_y} \xrightarrow[]{} \mathbb{R}^{n_x}$, is defined as
\begin{equation}
\begin{aligned}
    \GE(\HV) = \E_S \sup_{\boldsymbol h \in \HV} L_{\D}(\boldsymbol h) -  L_S(\boldsymbol h)
    \label{eq:Generalization_error_definition}
\end{aligned}
\end{equation}
where $L_{\D}(\boldsymbol h) = \E_{\D} \L(\boldsymbol h)$ is the expected loss with respect to the data distribution $\D$ (the joint probability distribution of the targets and observations), $L_{S}(\boldsymbol h) = \frac{1}{m} \sum_{i = 1}^m \L (\boldsymbol h(\boldsymbol y_i), \boldsymbol x_i) $ is the average empirical loss with respect to the training set $\mathcal{S}$, and $\E_S$ is the expectation over the training datasets.
The GE is a global property of the class of estimators, which captures how the class of estimators $\HV$ is suited to the learning problem.
Large GE implies that there are hypotheses in $\HV$ for which $L_D$ deviates much from $L_S$, on average over $S$.
However, the GE in (\ref{eq:Generalization_error_definition}) does not capture how the learning algorithm chooses the estimator $\boldsymbol h \in \HV$.

In order to capture the effect of the learning algorithm, we consider local properties of the class of estimators, and focus on bounding the estimation error (EE) 
\begin{equation}
\begin{aligned}
    \EE(\HV) = L_{\D}\left(\boldsymbol{\hat{h}} \right) - \inf_{\boldsymbol h \in \HV} L_{\D} \left( \boldsymbol h \right)
    \label{eq:Estimation_error_definition}
\end{aligned}
\end{equation}
where $\boldsymbol{\hat{h}}$ is the ERM satisfying $L_S (\boldsymbol{\hat{h}} ) = \inf_{\boldsymbol h \in \HV} L_{\S} (\boldsymbol h )$.
We note that the ERM approximates the learned estimator $\boldsymbol h_{\textit{learned}}$, which is obtained by training the network on a set of training examples $S$, using algorithms such as SGD.
However, the estimator $\boldsymbol h_{\textit{learned}}$ depends on the optimization algorithm, and can differ from the ERM.
The difference between the empirical loss associated with $\boldsymbol{\hat{ h}}$ and the empirical loss associated with $ \boldsymbol h_{\textit{learned}}$ is usually referred to as the optimization error.

Common deep neural network architectures have large GE compared to the low EE achieved in practice.
This still holds, when the networks are not trained with explicit regularization, such as weight decay, early stopping, or drop-outs \citep{Srivastava2014Dropout,Neyshabur2015c}.
This empirical phenomena is experienced across various architectures and hyper-parameter choices \citep{Liang2019Fisher,Novak2018,Lee2018Gauss,Neyshabur2018Overparametrization}.

\subsection{Rademacher Complexity based Bounds}
\label{sec:Rademacher Complexity Intro}

The RC is a standard tool which captures the ability of a class of functions to approximate noise, where a lower complexity indicates a reduced generalization error.
Formally, the empirical RC of a class of scalar estimators $\H$, such that $h: \mathbb{R}^{n_y} \xrightarrow[]{} \mathbb{R}$ for $h \in \H$, over $m$ samples is
\begin{equation}
\begin{aligned}
    \Ra{\H} = \E_{\{\epsilon_i\}_{i = 1}^m} \sup_{h \in \H} \frac{1}{m} \sum_{i = 1}^m \epsilon_{i} h(\boldsymbol y_i)
    \label{eq:Rademacher_definition}
\end{aligned}
\end{equation}
where $\{\epsilon_i\}_{i = 1}^m$ are independent Rademacher random variables for which $Pr\left(\epsilon_i = 1 \right) = Pr\left(\epsilon_i = -1 \right) = 1 / 2$, the samples $\{ \boldsymbol y_i \}_{i = 1}^m$ are obtained by the model in (\ref{eq:linear_inverse_problem}) from $m$ i.i.d. target vectors $\{ \boldsymbol x_i \}_{i = 1}^m$ drawn from an unknown distribution.
Taking the expectation of the RC of $\FH{\L}{\HV}$ with respect to the set of examples $S$ presented in Section \ref{sec:Network_Architecture}, leads to a bound on the GE
\begin{equation}
\begin{aligned}
    \GE(\HV) \leq 2\E_S \Ra{ \FH{\L}{\HV}}
    \label{eq:GE bound by Rademacher}
\end{aligned}
\end{equation}
where $\L$ is defined in Section \ref{sec:Network_Architecture}
and
$
    \FH{\L}{\HV} = \left\{ (\boldsymbol x, \boldsymbol y) \mapsto \L(h(\boldsymbol y),\boldsymbol x): h \in \HV \right\}
$
\citep{Shalev_Shwartz_book}.
We observe that the class of functions $\FH{\L}{\HV}$ consists of scalar functions, such that $f: \mathbb{R}^{n_x} \times \mathbb{R}^{n_x} \xrightarrow[]{} \mathbb{R}$ for $f \in \FH{\L}{\HV}$.
Therefore, in order to bound the GE of the class of functions defined by ISTA and ADMM networks, we can first bound their RC.

\label{sec:GE of ReLU}

Throughout the paper, we compare the model-based architectures with a feed forward network with ReLU activations, given by $\textit{ReLU}(\boldsymbol h) = \max \left(\boldsymbol h, \boldsymbol 0 \right)$.
In this section, we review existing bounds for the generalization error of these networks.
The layers of a ReLU network $\boldsymbol h_R^{l},\ \forall l \in [1, L]$, are defined by the following recurrence relation
$
    \boldsymbol h_R^{l} = \textit{ReLU} \left( \boldsymbol W^{l} \boldsymbol h_R^{l - 1} + \boldsymbol b\right)$,
and
$\boldsymbol h_R^0 = \textit{ReLU} (\boldsymbol b),
    \label{eq:ReLU_network}
$
where $\boldsymbol W^{l},\ \forall l \in [1,L]$ are the weight matrices which satisfy the same conditions as the weight matrices of ISTA and ADMM networks.
The class of functions representing the output at depth $L$ in a ReLU network, is $\HV_R^L = \{\boldsymbol h_R^L ( \{ \boldsymbol W^l \}_{l = 1}^{L} ): ||\boldsymbol W^l||_{\infty} \leq B_l\ l \in [1, L],\ ||\boldsymbol W^1||_{2} \leq B_1 \}$.
This architecture leads to the following bound on the GE.
\begin{theorem} [Generalization error bound for ReLU networks \citep{Gao2016Dropout}]
\label{theorem: GE of ReLU network}
    Consider the class of feed forward networks of depth-$L$ with ReLU activations, $\HV_R^L$, as described in Section \ref{sec:GE of ReLU}, and $m$ i.i.d. training samples.
    Given a $1$-Lipschitz loss function, its GE satisfies
    $
        \GE{\left( \HV_R^L \right)} \leq 2 G^l_R
    $, where $G^l_R = \frac{B_0 \prod_{l = 1}^L B_l}{\sqrt{m}}$.
\end{theorem}

\begin{proof}
    Follows from applying the bound of the RC of ReLU neural networks from \citep{Gao2016Dropout} and combining it with (\ref{eq:GE bound by Rademacher}).
\end{proof}

The bound in Theorem \ref{theorem: GE of ReLU network} is satisfied for any feed forward network with $1$-Lipschitz nonlinear activations (including ReLU), and can be generalized for networks with activations with different Lipshcitz constants.
We show in Theorem \ref{theorem: Lower GE of ReLU network}, that the bound presented in Theorem \ref{theorem: GE of ReLU network} cannot be substantially improved for ReLU networks with the RC framework.
\begin{theorem} [Lower Rademacher complexity bound for ReLU networks \citep{Bartlett2017Spectral}]
\label{theorem: Lower GE of ReLU network}
    Consider the class of feed forward networks of depth-$L$ with ReLU activations, where the weight matrices have bounded spectral norm $||\boldsymbol W^l||_{\sigma} \leq B^{\prime}_l,\ l \in [1, L]$. 
    The dimension of the output layer is $1$, and the dimension of each non-output layer is at least $2$.
    Given $m$ i.i.d. training samples, there exists a $c$ such that
    $
        \Ra{\H^{\prime,L}_R} \geq c B_0 \prod_{l = 1}^L B^{\prime}_l
    $, 
    where $\H_R^{\prime,L} = \{ h_R^L ( \{ \boldsymbol W^l \}_{l = 1}^{L} ): ||\boldsymbol W^l||_{\sigma} \leq B^{\prime}_l\ \, l \in [1, L] \}$.
\end{theorem}
This result shows that using the RC framework the GE of ReLU networks behaves as the product of the weight matrices' norms $\prod_{l = 1}^L B_l$, as captured in Theorem \ref{theorem: GE of ReLU network}.
Theorem \ref{theorem: Lower GE of ReLU network}, implies that the dependence on the weight matrices' norms, cannot be substantially improved with the RC framework for ReLU networks.

\section{Generalization Error Bounds: Global Properties}
\label{sec:Generalization Error Bounds}

In this section, we derive theoretical bounds on the GE of learned ISTA and ADMM networks.
From these bounds we deduce design rules to construct ISTA and ADMM networks with a GE which does not increase exponentially with the number of layers.
We start by presenting theoretical guarantees on the RC of any class of functions, after applying the soft-thresholding operation.

Soft-thresholding is a basic block that appears in multiple iterative algorithms, and therefore is used as the nonlinear activation in many model-based networks.
It results from the proximal gradient of the $L_1$ norm \citep{Yonina2010ConvexOpt}.
We therefore start by presenting the following lemma which expresses how the RC of a class of functions is affected by applying soft-thresholding to each function in the class.
The proof is provided in the supplementary material.


\begin{lemma}[Rademacher complexity of soft-thresholding]
\label{lemma:ST}
    Given any class of scalar functions $\H$ where $h:\mathbb{R}^{n} \xrightarrow[]{} \mathbb{R},\ h \in \H$ for any integer $n \geq 1$, and $m$ i.i.d. training samples,
    \begin{equation}
    \begin{aligned}
        \label{eq:result ST}
        \Ra{\FH{S_{\lambda}}{\H} } \leq \Ra{\H} - \frac{\lambda T}{m},
    \end{aligned}
    \end{equation}
    where $T = \sum_{i = 1}^m T_i$ and  $T_j =  \E_{\{\epsilon_i\}_{i = 2}^m} \left(\mathbbm{1}_{h^{\star}(\boldsymbol y_j) > \lambda \land h'^{\star}(\boldsymbol y_j) < - \lambda}\right)$ such that
    $
         h^*, \, h'^* = \argmax_{h, h' \in \H} \left(h(\boldsymbol y_1)  - h'(\boldsymbol y_1) - 2 \lambda \mathbbm{1}_{h(\boldsymbol y_1) > \lambda \land h'(\boldsymbol y_1) < - \lambda} + \sum_{i = 2}^m \epsilon_i (\S_{\lambda}(h(\boldsymbol y_i)) + \S_{\lambda}(h'(\boldsymbol y_i)))\right). \nonumber
    $
\end{lemma}

The quantity $T$ is a non-negative value obtained during the proof, which depends on the networks' number of layers, underlying data distribution $\D$ and soft-threshold value $\lambda$.
As seen from (\ref{eq:result ST}), the value of $T$ dictates the reduction in RC due to soft-thresholding, where a reduction in the RC can also be expected. 
The value of $T$ increases as $\lambda$ decreases.
In the case that $\lambda T$ increases with $\lambda$, higher values of $\lambda$ further reduce the RC of the class of functions $\H$, due to the soft-thresholding.

We now focus on the class of functions representing the output of a neuron at depth $L$ in an ISTA network, $\HV_I^L$.
In the following theorem, we bound its GE using the RC framework and Lemma \ref{lemma:ST}.
The proof is provided in the supplementary material.
\begin{theorem} [Generalization error bound for ISTA networks]
\label{theorem: GE of ISTA network}
    Consider the class of learned ISTA networks of depth $L$ as described in (\ref{eq:ISTA_unfolded}), and $m$ i.i.d. training samples.
    Then there exist $T^{(l)}$ for $\ l \in [1, L]$ in the range
    $
        T^{(l)} \in \left[ 0, \min \left\{ \frac{m B_{l} G_I^{l - 1}}{\lambda}, m \right\} \right]
    $
    such that $
        \GE{\left( \HV_I^L \right)} \leq 2 \E_{S} G_I^L
    $,
    where
    \begin{equation}
    \begin{aligned}
        G_I^l = \left(\frac{B_0 \prod_{l' = 1}^l B_{l'}}{\sqrt{m}} - \frac{\lambda}{m} \sum_{l' = 1}^{l-1} T^{(l')} \prod_{j = l'+1}^{l} B_j  - \frac{\lambda T^{(l)}}{m} \right).
    \end{aligned}
    \end{equation}
    
\end{theorem}

\label{remark: value of T}
Next, we show that for a specific distribution, the expected value of $T^{(l)}$ is greater than $0$.
Under an additional bound on the expectation value of the estimators (specified in the supplementary material)
    \begin{equation}
    \begin{aligned}
    \label{eq:bound Tl2}
        \mathbb{E}_{S} \left( T^{(l)} \right) &\geq \max \left\{ m \left(1 - 2e^{-(c B_l b^{(l)} - \lambda)} \right), 0 \right\}
    \end{aligned}
    \end{equation}
where $b^{(l + 1)} = B_l b^{(l)} - \lambda$, $b^{(1)} = B_0$, and $c \in (0, 1]$.
Increasing $\lambda$ or decreasing $B$, will decrease the bound in (\ref{eq:bound Tl2}), since crossing the threshold is less probable.
Depending on $\lambda$ and $\{ B_l \}_{l=1}^{L}$ (specifically that $\lambda \leq c B_l b^{(l)} + \ln 2$), the bound in (\ref{eq:bound Tl2}) is positive, and enforces a non-zero reduction in the GE.
Along with Theorem \ref{theorem: GE of ISTA network}, this shows the expected reduction in GE of ISTA networks compared to ReLU netowrks. 
The reduction is controlled by the value of the soft threshold.

To obtain a more compact relation, we can choose the maximal matrices' norm $B = \max_{l \in [1,L]} B_l$, and denote $T = \min_{l  \in [1,L]} T^{(l)} \in [0, m]$ which leads to 
$
    \GE{\left( \HV_I^L \right)} \leq 2 \left( \frac{B_0 B^L}{\sqrt{m}} - \frac{\lambda \E_{S}(T)}{m} \frac{B^L - 1}{B - 1} \right).
    \label{eq:Rademacher complexity of ISTA networks Simplified}
$

Comparing this bound with the GE bound for ReLU networks presented in Theorem \ref{theorem: GE of ReLU network}, shows the expected reduction due to the soft thresholding activation.
This result also implies practical rules for designing low generalization error ISTA networks. 
We note that the network's parameters such as the soft-threshold value $\lambda$ and number of samples $m$, are predefined by the model being solved 
(for example, in ISTA, the value of $\lambda$ is chosen according to the singular values of $\boldsymbol A$).

We derive an implicit design rule from (\ref{eq:Rademacher complexity of ISTA networks Simplified}), for a nonincreasing GE, as detailed in Section \ref{sec: appendix A2 GE}.
This is done by restricting the matrices' norm to satisfy $
    B \leq 1 + \frac{\lambda \E_S(T)}{\sqrt{m} B_0}
$.
Moreover, these results can be extended to convolutional neural networks.
As convolution operations can be expressed via multiplication with a convolution matrix, the presented results are also satisfied in that case.

Similarly, we bound the GE of the class of functions representing the output at depth $L$ in an ADMM network, $\H_A^L$. 
The proof and discussion are provided in the supplementary material.
\begin{theorem} [Generalization error bound for ADMM networks]
    \label{theorem: GE of ADMM network}
    Consider the class of learned ADMM networks of depth $L$ as described in (\ref{eq:ADMM_unfolded}), and $m$ i.i.d. training samples.
    Then there exist $T^{(l)}$ for $l \in [1, L-1]$ in the interval
    $
        T^{(l)}\in \left[ 0, \min \left\{\frac{m \tilde{B}_{l} G_A^{l - 1}}{\tilde{\lambda}}, m  \right\}\right]
    $
    where
    $
        G_A^{l} = \frac{B_0 \prod_{l' = 1}^{l-1} \tilde{B_{l'}}}{\sqrt{m}} - \frac{\tilde{\lambda}}{m} \sum_{l' = 1}^{l-2} T^{(l')} \prod_{j = l'+1}^{l - 1} \tilde{B}_j - \frac{\tilde{\lambda} T^{(l-1)}}{m}
    $,
    and $\tilde{\lambda} = (1 + \gamma) \lambda$, $\tilde{B}_l = (1 + 2 \gamma) (B_l + 2), \ l \in [1,L]$, such that $
        \GE\left( \HV_A^L \right)\leq 2 \tilde{B}_L \E_{S} G_A^{L - 1}.
        \label{eq:Rademacher complexity of ADMM networks}
    $
\end{theorem}

We compare the GE bounds for ISTA and ADMM networks, to the bound on the GE of ReLU networks presented in Theorem \ref{theorem: GE of ReLU network}.
We observe that both model-based networks achieve a reduction in the GE, which depends on the soft-threshold, the underlying data distribution, and the bound on the norm of the weight matrices.
Following the bound, we observe that the soft-thresholding nonlinearity is most valuable in the case of small number of training samples.
The soft-thresholding nonlinearity is the key that enables reducing the GE of the ISTA and ADMM networks compared to ReLU networks.
Next, we focus on bounding the EE of feed-forward networks based on the LRC framework.

\section{Estimation Error Bounds: Local Properties}
\label{sec:Estimation Error Bounds}

To investigate the model-based networks' EE, we use the LRC framework \citep{Bartlett2005LRC}.
Instead of considering the entire class of functions $\HV$, the LRC considers only estimators which are close to the optimal estimator
$
    \HV_r = \left\{ \boldsymbol h \in \HV: \E_{\D} \norm{ \boldsymbol h(\boldsymbol y) - \boldsymbol h^*(\boldsymbol y) }^2_2 \leq r \right\}
    \label{def:Ar}
$,
where $\boldsymbol h^*$ is such that $L_D ( \boldsymbol h^* ) = \inf_{\boldsymbol h \in \HV} L_{\D}(\boldsymbol h)$.
It is interesting to note that the class of estimators $\HV_r$ only restricts the distance between the estimators themselves, and not between their corresponding losses.
Following the LRC framework, we consider target vectors ranging in $[-1, 1]^{n_x}$.
Therefore, we adapt the networks' estimations by clipping them to lie in the interval $[-1, 1]^{n_x}$.
In our case we consider the restricted classes of functions representing the output of a neuron at depth $L$ in ISTA, ADMM, and ReLU networks.
Moreover, we denote by $\boldsymbol W^l$ and $\boldsymbol W^{l,*}, \ l \in [1, L]$ the weight matrices corresponding to $\boldsymbol h \in \HV_r$ and $\boldsymbol h^*$, respectively. 
Based on these restricted class of functions, we present the following assumption and theorem (the proof is provided in the supplementary material).

\begin{assumption}
    \label{assumption: loss}
        
    There exists a constant $C \geq 1$ such that for every probability distribution $\D$, and estimator $\boldsymbol h \in \HV$,  $
            \E_{\D} \sum_{j = 1}^{n_x} (\boldsymbol h_j - \boldsymbol h^*_j )^2 \leq C \E_{\D} \sum_{j = 1}^{n_x} \left( \ell(\boldsymbol h_j) - \ell (\boldsymbol h^*_j ) \right)$,
    where $\boldsymbol h_j$ and $\boldsymbol h^*_j$ denote the $j$th coordinate of the estimators.
\end{assumption}
As pointed out in \citep{Bartlett2002Rademacher}, this condition usually follows from a uniform convexity condition on the loss function $\ell$.  For instance, if $\left|h(\boldsymbol y) - \boldsymbol x \right| \leq 1$ for any $h \in \H, \boldsymbol y \in \R^{n_y}$ and $\boldsymbol x \in \R^{n_x}$,  then the condition is satisfied with $C=1$ \citep{Yousefi2018LRCMTL}.
\begin{theorem} [Estimation error bound of ISTA, ADMM, and ReLU  networks]
\label{theorem: EE of ISTA network}
Consider the class of functions represented by depth-$L$ ISTA networks $\HV_I^L$ as detailed in Section \ref{sec:Network_Architecture}, $m$ i.i.d training samples, and a per-coordinate loss satisfying Assumption \ref{assumption: loss} with a constant $C$.
    Let $|| \boldsymbol W^{l} -\boldsymbol W^{l,*} ||_{\infty} \leq \alpha \sqrt{r}$ for some $\alpha > 0$.
    Moreover, $B \geq \max\{\alpha \sqrt{r}, 1\}$.
    Then there exists $T$ in the interval
    $
        T \in \left[ 0, \min \left\{ \frac{\sqrt{m} B_0 B^{L - 1} 2^L }{\lambda \eta}, m\right\} \right]
    $,
    where $\eta = \frac{L B^{L-1} (B-1) - B^L + 1}{(B-1)^2}$, such that for any $s > 0$ with probability at least $1 - e^{-s}$,
    \begin{equation}
    \begin{aligned}
        \label{eq:EE bound ISTA network}
        \EE{\left( \HV_I^L \right)} \leq 41 r^* + \frac{17 C^2 + 48 C}{m n_x} s
    \end{aligned}
    \end{equation}
    where $
        r^* = C^2 \alpha^2 \left(\frac{B_0 B^{L-1}2^L}{\sqrt{m}} - \frac{\lambda T}{m} \eta \right)^2$.
    The bound is also satisfied for the class of functions represented by depth-$L$ ADMM networks $\HV_A^L$, with $
        r^* = C^2 \alpha^2 \left(\frac{B_0  \tilde{B}^{L-2}2^{L-1}}{\sqrt{m}} - \frac{\tilde{\lambda} T}{m} \tilde{\eta} \right)^2
    $,
    where $\tilde{\lambda} = (1 + \gamma) \lambda$, $\tilde{B} = (1 + 2 \gamma) (B + 2)$, and $\tilde{\eta} = \frac{(L-1) \tilde{B}^{L-2} (\tilde{B}-1) - \tilde{B}^{L-1} + 1}{(\tilde{B}-1)^2}$,
    and for the class of functions represented by depth-$L$ ReLU networks $\HV_R^L$, with $
        r^* =  C^2 \alpha^2 \, \left(\frac{B_0 B^{L-1}2^L}{\sqrt{m}}\right)^2$.

\end{theorem}





From Theorem \ref{theorem: EE of ISTA network}, we observe that the EE decreases by a factor of $\O \left(1 / m \right)$, instead of a factor of $\O \left(1 / \sqrt{m}\right)$ obtained for the GE.
This result complies with previous local bounds which yield faster convergence rates compared to global bounds \citep{Blanchard2007localVSglobal,Bartlett2005LRC}. 
Also, the value of $\alpha$ relates the maximal distance between estimators in $\H_r$ denoted by $r$, to the distance between their corresponding weight matrices $|| \boldsymbol W^{l} -\boldsymbol W^{l,*} ||_{\infty} \leq \alpha \sqrt{r},\ l \in [1, L]$.
Tighter bounds on the distance between the weight matrices allow us to choose a smaller value for $\alpha$, resulting in smaller values $r^*$ which improve the EE bounds.
The value of $\alpha$ could depend on the network's nonlinearity, underlying data distribution $\D$, and number of layers.

Note that the bounds of the model-based architectures depend on the soft-thresholding through the value of $\lambda \E_S(T)$.
As $\lambda \E_S(T)$ increases, the bound on the EE decreases, which emphasizes the nonlinearity's role in the network's generalization abilities.
Due to the soft-thresholding, ISTA and ADMM networks result in lower EE bounds compared to the bound for ReLU networks.
It is interesting to note, that as the number of training samples $m$ increases, the difference between the bounds on the model-based and ReLU networks is less significant.
In the EE bounds of model-based networks, the parameter $B_0$ relates the bound to the sparsity level $\rho$, of the target vectors.
Lower values of $\rho$ result in lower EE bounds, as demonstrated in the supplementary.

\section{Numerical Experiments}
\label{sec:Simulations}

In this section, we present a series of experiments that concentrate on how a particular model-based network (ISTA network) compares to a ReLU network, 
and showcase the merits of model-based networks.
We focus on networks with $10$ layers (similar to previous works \citep{GregorLeCun2010LISTA}), to represent realistic model-based network architectures.
The networks are trained on a simulated dataset to solve the problem in (\ref{eq:linear_inverse_problem}), with target vectors uniformly distributed in $[-1, 1]$.
The linear mapping $\boldsymbol A$ is constructed from the real part of discrete Fourier transform (DFT) matrix rows \citep{Ong2019ISTA_DFT}, where the rows are randomly chosen.
The sparsity rate is $\rho = 0.15$, and the noise's standard deviation is $0.1$.
To train the networks we used the SGD optimizer with the $L_1$ loss over all neurons of the last layer.
The target and noise vectors are generated as element wise independently from a uniform distribution ranging in $[-1, 1]$ .
All results are reproducible through \citep{CodeExperimentsGit} which provides the complete code to execute the experiments presented in this section.

We concentrate on comparing the networks' EE, since in practice, the networks are trained with a finite number of examples.
In order to empirically approximate the EE of a class of networks $\H$, we use an empirical approximation of $h^*$ (which satisfy $\L_{\D} (h^*) = \inf_{h \in \H} \L_{\D}(h)$) and the ERM $\hat{h}$, denoted by $h^*_{emp}$ and $\hat{h}_{emp}$ respectively.
The estimator $h^*_{emp}$ results from the trained network with $10^4$ samples, and the ERM is approximated by a network trained using SGD with $m$ training samples (where $m \leq 10^4$).
The empirical EE is given by their difference $\L_{\D}(\hat{h}_{emp}) - \L_{\D}(h^*_{emp})$.

In Fig. \ref{fig:ISTA vs ReLU depth 10}, we compare between the ISTA and ReLU networks, in terms of EE and $L_1$ loss, for networks trained with different number of samples (between $10$ and $10^4$ samples).
We observe that for small number of training samples, the ISTA network substantially reduces the EE compared to the ReLU network.
This can be understood from Theorem \ref{theorem: EE of ISTA network}, which results in lower EE bounds on the ISTA networks compared to ReLU networks, due to the term $\lambda \E_S(T)$. 
However, for large number of samples the EE of both networks decreases to zero, which is also expected from Theorem \ref{theorem: EE of ISTA network}.
This highlights that the contribution of the soft-thresholding nonlinearity to the generalization abilities of the network is more significant for small number of training samples.
Throughout the paper, we considered networks with constant bias terms.
In this section, we also consider learned bias terms, as detailed in the supplementary material.
In Fig. \ref{fig:ISTA vs ReLU depth 10}, we present the experimental results for networks with constant and learned biases. 
The experiments indicate that the choice of constant or learned bias is less significant to the EE or the accuracy, compared to the choice of nonlinearity, emphasizing the relevance of the theoretical guarantees.
The cases of learned and constant biases have different optimal estimators, as the networks with learned biases have more learned parameters.
As a result, it is plausible a network with more learnable parameters (the learnable bias) exhibits a lower estimation error since the corresponding optimal estimator also exhibits a lower estimation error.
\begin{figure}[!h]
        \centering
        \includegraphics[width = 3.0in]{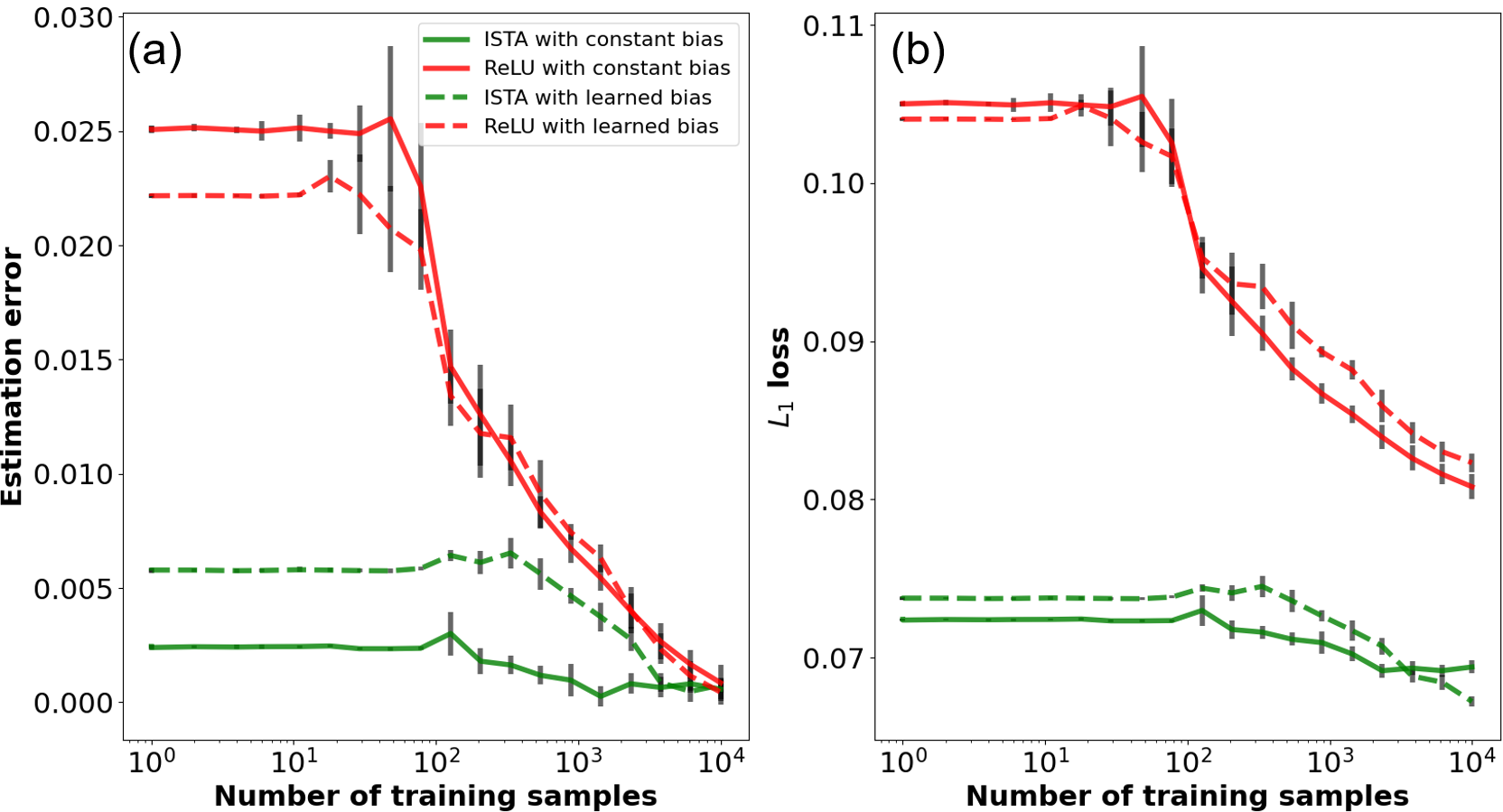}
        \caption{Comparing the EE of ISTA and ReLU networks with $10$ layers. \textbf{a.} The EE of the networks as a function of the training samples. \textbf{b.} The $L_1$ of the networks as a function of the training samples.
        The ISTA network achieves lower EE compared to the ReLU network, along with lower losses.}
    \label{fig:ISTA vs ReLU depth 10}
\end{figure}
To analyze the effect of the soft-thresholding value on the generalization abilities of the ISTA network, we show in Fig. \ref{fig:ISTA  lambda}, the empirical EE for multiple values of $\lambda$.
The experimental results demonstrate that for small number of samples, increasing $\lambda$ reduces the EE.
As expected from the EE bounds, for a large number of training samples, this dependency on the nonlinearity vanishes, and the EE is similar for all values of $\lambda$.
In Fig. \ref{fig:ISTA lambda}b, we show the $L_1$ loss of the ISTA networks for different values of $\lambda$.
We observe that low estimation error does not necessarily lead to low loss value. For $m \lesssim 100 $, increasing $\lambda$ reduced the EE.
These results suggest that given ISTA networks with different values of $\lambda$ such that all networks achieve similar accuracy, the networks with higher values of $\lambda$ provide lower EE.
These results are also valid for additional networks' depths.
In Section \ref{sec:Additional simulation results},  we compare the EE for networks with different number of layers, and show that they exhibit a similar behaviour.
\begin{figure}[!h]
        \centering
        \includegraphics[width = 3.0in]{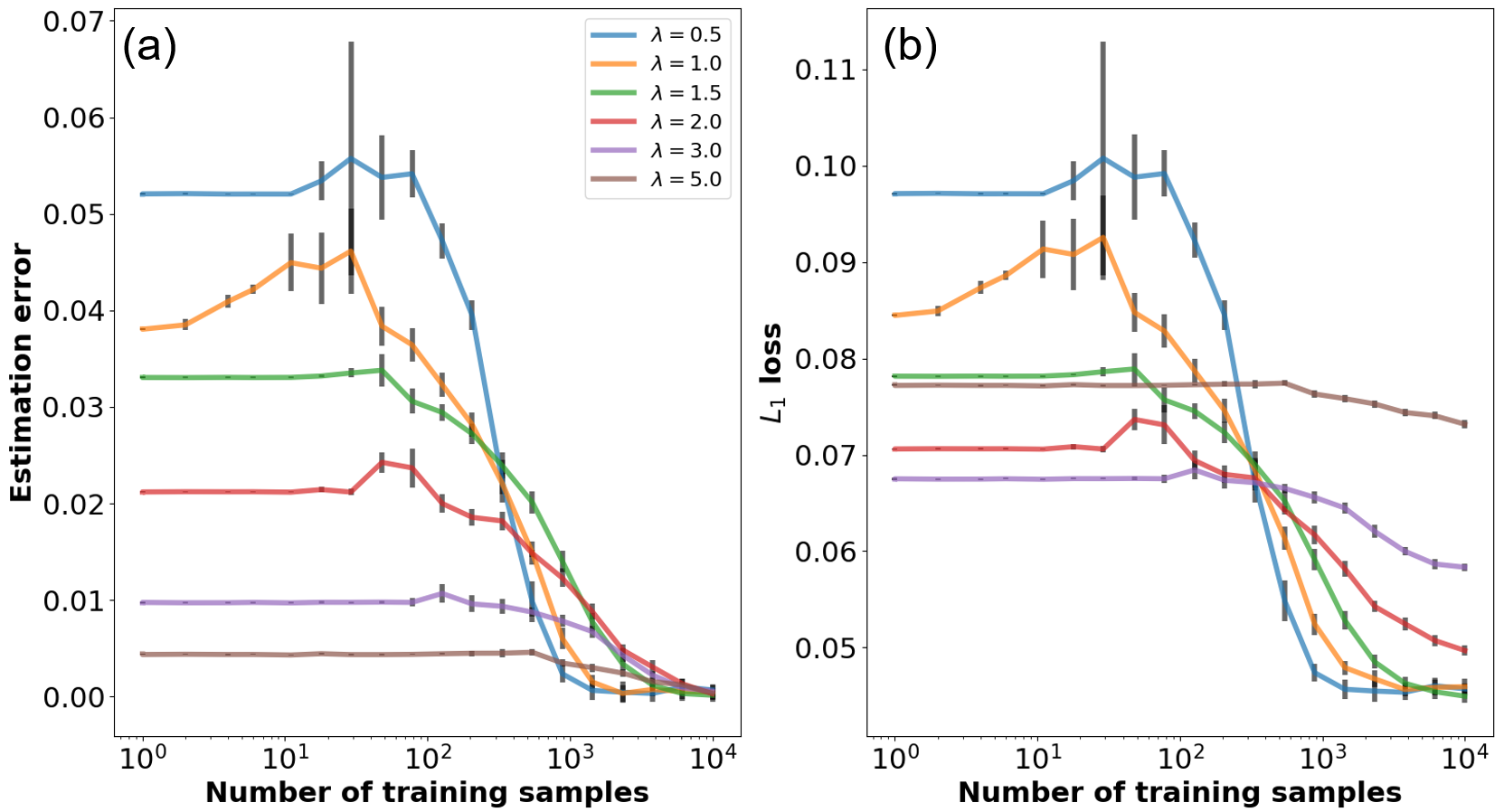}
        \caption{Estimation error and loss of ISTA networks with $10$ layers, as a function of the soft-threshold's value $\lambda$.
        The experiments indicate that the estimation error decreases as $\lambda$ increases, which demonstrates a way to control the networks' generalization abilities through $\lambda$.}
    \label{fig:ISTA lambda}
\end{figure}

\section{Conclusion}
\label{sec:conclusions}

We derived new GE and EE bounds for ISTA and ADMM networks, based on the RC and LRC frameworks.
Under suitable conditions, the model-based networks' GE is nonincreasing with depth, resulting in a substantial improvement compared to the GE bound of ReLU networks.
The EE bounds explain EE behaviours experienced in practice, such as ISTA networks demonstrating higher estimation abilities, compared to ReLU networks, especially for small number of training samples.
Through a series of experiments, we show that the generalization abilities of ISTA networks are controlled by the soft-threshold value, and achieve lower EE along with a more accurate recovery compared to ReLU networks which increase the GE and EE with the networks' depths.

It is interesting to consider how the theoretical insights can be harnessed to enforce neural networks with high generalization abilities.
One approach is to introduce an additional regularizer during the training process that is rooted in the LRC, penalizing networks with high EE \citep{Yang2019}.  

\bibliography{iclr2023_conference}
\bibliographystyle{iclr2023_conference}

\newpage

\appendix
\section{Supplementary Material}

In this supplementary material we provide the proofs of the presented lemmas and theorems.

\subsection{Rademacher complexity of the soft-thresholding operator}
\label{sec: appendix lemma ST}
In this section, we prove Lemma \ref{lemma:ST} which examines how the RC is affected by the soft-thresholding operator. 

\textbf{Lemma \ref{lemma:ST}} (Rademacher complexity of soft-thresholding): \\
    Given any class of scalar functions $\H$ where $h:\mathbb{R}^{n} \xrightarrow[]{} \mathbb{R},\ h \in \H$ for any integer $n \geq 1$, there exists a value $T$ in the interval
    \begin{equation}
    \begin{aligned}
        T \in \left(0, \min \left\{ \frac{m \Ra{\H }}{\lambda}, m \right\} \right]
    \end{aligned}
    \end{equation}
    such that
    \begin{equation}
    \begin{aligned}
        \Ra{\FH{S_{\lambda}}{\H} } \leq \Ra{\H} - \frac{\lambda T}{m}.
    \end{aligned}
    \end{equation}
    
\begin{proof}
    We start by denoting 
    \begin{equation}
    \begin{aligned}
        R = m \Ra{\FH{S_{\lambda}}{\H} } = \E_{\{\epsilon_i\}_{i = 1}^m} \sup_{h \in \H} \sum_{i = 1}^m  \epsilon_i \S_{\lambda}( h(\boldsymbol y_i))
    \end{aligned}
    \end{equation}
    and we explicitly compute the expectation on $\epsilon_1$:
    \begin{equation}
    \begin{aligned}
        \label{eq:R_helper}
        R = \E_{\{\epsilon_i\}_{i = 1}^m} \ \ \ & \sup_{h \in \H} \sum_{i = 1}^m  \epsilon_i \S_{\lambda}( h(\boldsymbol y_i)) \\
        = \E_{\{\epsilon_i\}_{i = 1}^m} \ \ \ & \sup_{h \in \H} \epsilon_1 \S_{\lambda}(h(\boldsymbol y_1)) + \sum_{i = 2}^m \epsilon_i \S_{\lambda}(h(\boldsymbol y_i)) \\
        = \E_{\{\epsilon_i\}_{i = 2}^m} \ \ \ & \left[ \frac{1}{2} \sup_{h \in \H} \left(\S_{\lambda}(h(\boldsymbol y_1)) + \sum_{i = 2}^m \epsilon_i \S_{\lambda}(h(\boldsymbol y_i)) \right) + \frac{1}{2} \sup_{h' \in \H} \left(-\S_{\lambda}(h'(\boldsymbol y_1)) + \sum_{i = 2}^m \epsilon_i \S_{\lambda}(h'(\boldsymbol y_i)) \right) \right]\\
        = \frac{1}{2} \E_{\{\epsilon_i\}_{i = 2}^m} & \sup_{h \in \H} \left(\S_{\lambda}(h(\boldsymbol y_1)) + \sum_{i = 2}^m \epsilon_i \S_{\lambda}(h(\boldsymbol y_i)) \right) + \sup_{h' \in \H} \left(-\S_{\lambda}(h'(\boldsymbol y_1)) + \sum_{i = 2}^m \epsilon_i \S_{\lambda}(h'(\boldsymbol y_i)) \right) \\
        = \frac{1}{2} \E_{\{\epsilon_i\}_{i = 2}^m} & \sup_{h, h' \in \H} \left(\S_{\lambda}(h(\boldsymbol y_1)) - \S_{\lambda}(h'(\boldsymbol y_1)) + \sum_{i = 2}^m \epsilon_i (\S_{\lambda}(h(\boldsymbol y_i)) + \S_{\lambda}(h'(\boldsymbol y_i)))\right).
    \end{aligned}
    \end{equation}
    To simplify the notations, we denote
    \begin{equation}
    \begin{aligned}
        \sh{h}{h'} & = \sum_{i = 2}^m \epsilon_i (\S_{\lambda}(h(\boldsymbol y_i)) + \S_{\lambda}(h'(\boldsymbol y_i)))
    \end{aligned}
    \end{equation}
    and observe that $\sh{h}{h'} = \sh{h'}{h}$.
    
    Next, we show that the soft-thresholding on the estimations of the first sample, reduces the RC, and results in
    \begin{equation}
    \begin{aligned}
    \label{eq:r helper}
        R \leq \frac{1}{2} \E_{\{\epsilon_i\}_{i = 2}^m} & \sup_{h, h' \in \H} \left(h(\boldsymbol y_1)  - h'(\boldsymbol y_1) - 2 \lambda \mathbbm{1}_{h(\boldsymbol y_1) > \lambda \land h'(\boldsymbol y_1) < - \lambda} + \sh{h}{h'}\right).
    \end{aligned}
    \end{equation}

    The soft-thresholding function is piecewise linear with three pieces, and can be written as
    \begin{equation}
    \begin{aligned}
        \label{eq:ST_regions}
        \S_{\lambda}(x) = \begin{cases}
			x + \lambda, & x < -\lambda \\
            0, & x \in [-\lambda, \lambda] \\
            x - \lambda, & x > \lambda
		 \end{cases}.
    \end{aligned}
    \end{equation}
    We focus on the terms $h(\boldsymbol y_1)$ and $h'(\boldsymbol y_1)$, in the different regions determined by (\ref{eq:ST_regions}).
    To capture these regions, we define the following classes of functions
    \begin{equation}
    \begin{aligned}
			\H^{-1} &= \left\{ h \in \H: h(\boldsymbol y_1) < - \lambda \right\} \\
            \H^0 &= \left\{ h \in \H: h(\boldsymbol y_1) \in [-\lambda, \lambda] \right\} \\
            \H^1 &= \left\{ h \in \H: h(\boldsymbol y_1) > \lambda \right\}.
    \end{aligned}
    \end{equation}
    We can then write
    \begin{equation}
    \begin{aligned}
        \label{eq: r alpha beta}
        r = \max_{\alpha, \beta \in \{ -1, 0, 1\}} r_{\alpha, \beta}
    \end{aligned}
    \end{equation}
    where
    \begin{equation}
    \begin{aligned}
        r_{\alpha, \beta} = \sup_{h \in \H^{\alpha} , h' \in \H^{\beta}} \left(\S_{\lambda}(h(\boldsymbol y_1)) - \S_{\lambda}(h'(\boldsymbol y_1)) + \sh{h}{h'} \right).
    \end{aligned}
    \end{equation}
    
    We now divide into cases for the different values of $\alpha$ and $\beta$:
    \begin{enumerate}
        \item For $\alpha, \beta = 1$ or $\alpha, \beta = -1$, we obtain that
        $\S_{\lambda} ( h(\boldsymbol y_1)) - \S_{\lambda} (h'(\boldsymbol y_1) ) = h(\boldsymbol y_1) - h'(\boldsymbol y_1)$. 
        As a result,
        \begin{equation}
        \begin{aligned}
        r_{1, 1} = \sup_{h \in \H^{1} , h' \in \H^{1}} \left( h(\boldsymbol y_1) - h'(\boldsymbol y_1) + \sh{h}{h'} \right)
        \end{aligned}
        \end{equation}
        and
        \begin{equation}
        \begin{aligned}
        r_{-1, -1} = \sup_{h \in \H^{-1} , h' \in \H^{-1}} \left( h(\boldsymbol y_1) - h'(\boldsymbol y_1) + \sh{h}{h'} \right).
        \end{aligned}
        \end{equation}
        
        \item For $\alpha = 1$ and $\beta = 0$, we obtain that $\S_{\lambda} ( h(\boldsymbol y_1)) - \S_{\lambda} (h'(\boldsymbol y_1) ) = h(\boldsymbol y_1) - \lambda$.
        In this case, $h'(\boldsymbol y_1) \leq \lambda$, which results in
        \begin{equation}
        \begin{aligned}
            r_{1, 0} & = \sup_{h \in \H^{1} , h' \in \H^{0}} \left( h(\boldsymbol y_1) - \lambda + \sh{h}{h'} \right) \\
            & \leq \sup_{h \in \H^{1} , h' \in \H^{0}} \left( h(\boldsymbol y_1) - h'(\boldsymbol y_1) + \sh{h}{h'} \right).
        \end{aligned}
        \end{equation}
        
        \item For $\alpha = 0$ and $\beta = -1$, we obtain that $\S_{\lambda} ( h(\boldsymbol y_1)) - \S_{\lambda} (h'(\boldsymbol y_1) ) = -h'(\boldsymbol y_1) - \lambda$.
        In this case, $h(\boldsymbol y_1) \geq -\lambda$, which results in
        \begin{equation}
        \begin{aligned}
            r_{0, -1} & = \sup_{h \in \H^{0} , h' \in \H^{-1}} \left( -h'(\boldsymbol y_1) - \lambda + \sh{h}{h'} \right) \\
            & \leq \sup_{h \in \H^{0} , h' \in \H^{-1}} \left( h(\boldsymbol y_1) - h'(\boldsymbol y_1) + \sh{h}{h'} \right).
        \end{aligned}
        \end{equation}

        \item For $\alpha = 1$ and $\beta = -1$, we obtain that $\S_{\lambda} ( h(\boldsymbol y_1)) - \S_{\lambda} (h'(\boldsymbol y_1) ) = h(\boldsymbol y_1) - h'(\boldsymbol y_1) - 2\lambda$, which results in
        \begin{equation}
        \begin{aligned}
            r_{1, -1} & = \sup_{h \in \H^{1} , h' \in \H^{-1}} \left( h(\boldsymbol y_1) - h'(\boldsymbol y_1) - 2 \lambda + \sh{h}{h'} \right).
        \end{aligned}
        \end{equation}
        In this case, we obtain a reduction by $2 \lambda$ in the complexity.
        
        \item For $\alpha, \beta = 0$, we obtain that
        $\S_{\lambda} ( h(\boldsymbol y_1)) - \S_{\lambda} (h'(\boldsymbol y_1) ) = 0$, which leads to
        \begin{equation}
        \begin{aligned}
        r_{0, 0} = \sup_{h \in \H^{0} , h' \in \H^{0}} \sh{h}{h'}.
        \end{aligned}
        \end{equation}
        We bound this term to obtain a similar expression to the previous cases.
        Any pair of estimators $h,h' \in \H^{0}$ satisfies that $h(\boldsymbol y_1) \geq h'(\boldsymbol y_1)$ or $h(\boldsymbol y_1) \leq h'(\boldsymbol y_1)$.
        Using the fact that $\sh{h}{h'} = \sh{h'}{h}$, there exists a pair $h,h' \in \H^{0}$ satisfying $h(\boldsymbol y_1) - h'(\boldsymbol y_1) + \sh{h}{h'} \geq \sh{h}{h'}$ (otherwise, interchanging between the estimators satisfies this requirement since $h'(\boldsymbol y_1) - h(\boldsymbol y_1) + \sh{h}{h'} \geq \sh{h}{h'} = \sh{h'}{h}$, where the estimators satisfy $h',h \in \H^0$).
        This means that the term $r_{0,0}$ is bounded by
        \begin{equation}
        \begin{aligned}
            r_{0, 0} \leq \sup_{h \in \H^{0} , h' \in \H^{0}} \left( h(\boldsymbol y_1) - h'(\boldsymbol y_1) + \sh{h}{h'} \right).
        \end{aligned}
        \end{equation}
        
    \end{enumerate}
    
    We will now show that the additional cases that were not presented above, are dominated by other cases (i.e. upper bounded), and therefore will not be considered in (\ref{eq: r alpha beta}):
    \begin{enumerate}
        \item For $\alpha = -1$ and $\beta = 1$, we obtain that $\S_{\lambda} ( h(\boldsymbol y_1)) - \S_{\lambda} (h'(\boldsymbol y_1) ) = h(\boldsymbol y_1) - h'(\boldsymbol y_1) + 2\lambda$, which results in:
        \begin{equation}
        \begin{aligned}
            r_{-1, 1} & = \sup_{h \in \H^{-1} , h' \in \H^{1}} \left( h(\boldsymbol y_1) - h'(\boldsymbol y_1) + 2 \lambda + \sh{h}{h'} \right).
        \end{aligned}
        \end{equation}
        We now show that $r_{-1, 1} \leq r_{1, -1}$, and therefore the case of $\alpha = -1$ and $\beta = 1$ is upper bounded by the case of $\alpha = 1$ and $\beta = -1$.
        For any pair of estimators in $h \in \H^{-1}$ and $h' \in \H^{1}$, there exists a pair of estimators in $h \in \H^{1}$ and $h' \in \H^{-1}$, which achieve a higher value (showing that $r_{-1, 1} \leq r_{1, -1}$).
        Since we consider the case where $\alpha = -1$ and $\beta = 1$, there exist $\delta(h), \delta'(h') > 0$ such that $h(\boldsymbol y_1) = -\delta(h) - \lambda$ and $h'(\boldsymbol y_1) = \delta'(h') + \lambda$, resulting in:
        \begin{equation}
        \begin{aligned}
            r_{-1,1} = \sup_{h \in \H^{-1} , h' \in \H^{1}} c_{-1,1}(h,h')
        \end{aligned}
        \end{equation}        
        where
        \begin{equation}
        \begin{aligned}
            c_{-1,1}(h,h') & = -\delta(h) -\delta'(h') + \sh{h}{h'}.
        \end{aligned}
        \end{equation}
        Similarly, for the case of $\alpha = 1$ and $\beta = -1$, we can write
        \begin{equation}
        \begin{aligned}
            r_{1,-1} = \sup_{h \in \H^{1} , h' \in \H^{-1}} c_{1,-1}(h,h')
        \end{aligned}
        \end{equation}        
        such that
        \begin{equation}
        \begin{aligned}
            c_{1,-1}(h,h') & = \lambda + \delta(h) + \lambda + \delta'(h') -2 \lambda + \sh{h}{h'}  \\
            & = \delta(h) + \delta'(h') + \sh{h}{h'}.
        \end{aligned}
        \end{equation}
        Since $\sh{h'}{h} = \sh{h}{h'}$, for any estimators $h_{-1} = -\delta_{-1}(h_{-1}) - \lambda \in \H^{-1}$ and $h_{1} = \delta_1( h_1) + \lambda \in \H^{1}$
        \begin{equation}
        \begin{aligned}
             c_{1,-1}(h_1,h_{-1}) &= \delta_1(h_1) + \delta_2(h_{-1}) + \sh{h_1}{h_{-1}}\\
             &\geq -\delta_1(h_1) - \delta_2(h_{-1}) + \sh{h_{-1}}{h_1}\\
             &= c_{-1,1}(h_{-1},h_1)
        \end{aligned}
        \end{equation}                
        implying that
        \begin{equation}
        \begin{aligned}
             r_{-1,1} = \sup_{h \in \H^{-1} , h' \in \H^{1}} c_{-1,1}(h,h') \leq \sup_{h' \in \H^{1} , h \in \H^{-1}} c_{1,-1}(h',h) = r_{1,-1}.
        \end{aligned}
        \end{equation}

        \item For $\alpha = -1$ and $\beta = 0$, we obtain that $\S_{\lambda} ( h(\boldsymbol y_1)) - \S_{\lambda} (h'(\boldsymbol y_1) ) = h(\boldsymbol y_1) + \lambda$, which results in:
        \begin{equation}
        \begin{aligned}
            r_{-1, 0} = \sup_{h \in \H^{-1} , h' \in \H^{1}} c_{-1,0}(h, h')
        \end{aligned}
        \end{equation}
        where
        \begin{equation}
        \begin{aligned}
            c_{-1, 0}(h,h') & = h(\boldsymbol y_1) + \lambda + \sh{h}{h'}.
        \end{aligned}
        \end{equation}
        We will show that $r_{-1, 0} \leq r_{0, -1}$, and therefore the case of $\alpha = -1$ and $\beta = 0$ is dominated (upper bounded) by the case of $\alpha = 0$ and $\beta = -1$.
        Since we consider the case where $\alpha = -1$ and $\beta = 0$, there exist $\delta(h) > 0$ such that $h(\boldsymbol y_1) = -\delta(h) - \lambda$, resulting in:
        \begin{equation}
        \begin{aligned}
            c_{-1,0}(h,h') & = -\delta(h) + \sh{h}{h'}.
        \end{aligned}
        \end{equation}
        Similarly, for $\alpha = 0$ and $\beta = -1$, we can write
        \begin{equation}
        \begin{aligned}
            r_{0, -1} = \sup_{h \in \H^{0} , h' \in \H^{-1}} c_{0,-1}(h, h')
        \end{aligned}
        \end{equation}
        where
        \begin{equation}
        \begin{aligned}
            c_{0,-1}(h,h') &= \lambda + \delta(h) - \lambda + \sh{h'}{h} \\
            & = \delta(h) + \sh{h}{h'}.
        \end{aligned}
        \end{equation}
        Again, we consider any estimators $h_{-1} \in H^{-1}$ and $h_0 \in \H^0$
        \begin{equation}
        \begin{aligned}
            c_{0,-1}(h_0,h_{-1}) &= \delta_{-1}(h_{-1}) + \sh{h_0}{h_{-1}} \\
            & \geq -\delta_{-1}(h_{-1}) + \sh{h_{-1}}{h_0}\\
            & =  c_{-1,0}(h_{-1},h_0).
        \end{aligned}
        \end{equation}
        This implies that for any pair of estimators in $\H^{-1}$ and $\H^{0}$, there exists a pair of estimators in $\H^{0}$ and $\H^{-1}$, which achieve a higher value, proving that
        \begin{equation}
        \begin{aligned}
            r_{-1, 0} = \sup_{h \in \H^{-1} , h' \in \H^{0}} c_{-1,0}(h, h') \leq \sup_{h' \in \H^{0} , h \in \H^{-1}} c_{0,-1}(h', h) = r_{0, -1}.
        \end{aligned}
        \end{equation}
        
        Following the same chain of thought, one can prove that $r_{0, 1} \leq r_{1, 0}$, which makes the case of $\alpha = 0$ and $\beta = 1$ dominated by that of $\alpha = 1$ and $\beta = 0$.
    \end{enumerate}
    
    Taking into account the dominant cases, we can re-write (\ref{eq: r alpha beta}) as
    \begin{equation}
    \begin{aligned}
        \label{eq: r alpha beta2}
        r & = \max_{\alpha, \beta \in \{ -1, 0, 1\}} r_{\alpha, \beta} 
        & = \max_{( \alpha, \beta ) \in \mathcal{I}} r_{\alpha, \beta}
    \end{aligned}
    \end{equation}
    where 
    \begin{equation}
    \begin{aligned}
        \mathcal{I} = \left\{ \begin{pmatrix} 1, 1 \end{pmatrix},
        \begin{pmatrix} 1, 0 \end{pmatrix},
        \begin{pmatrix} 1, -1 \end{pmatrix},
        \begin{pmatrix} 0, 0 \end{pmatrix},
        \begin{pmatrix} 0, -1 \end{pmatrix}\right\}
    \end{aligned}
    \end{equation}
    is the set of dominant cases.
    We observe, that all the dominant cases, $\begin{pmatrix} \alpha, \beta \end{pmatrix} \in \mathcal{I}$, are bounded by
    \begin{equation}
    \begin{aligned}
        r_{\alpha, \beta} \leq \sup_{h \in \H^{\alpha} , h' \in \H^{\beta}} \left( h(\boldsymbol y_1) - h'(\boldsymbol y_1) - 2 \lambda \mathbbm{1}_{h(\boldsymbol y_1) > \lambda \land h'(\boldsymbol y_1) < - \lambda} + \sh{h}{h'} \right)
    \end{aligned}
    \end{equation}
    where the indicator takes into account the reduction in complexity obtained for the case of $\alpha = 1$ and $\beta = -1$.
    Since $\H^{\alpha} \subseteq \H, \alpha \in \{ -1, 0, 1\}$, the above is bounded by
    \begin{equation}
    \begin{aligned}
        r_{\alpha, \beta} \leq \sup_{h \in \H, h' \in \H} \left( h(\boldsymbol y_1) - h'(\boldsymbol y_1) - 2 \lambda \mathbbm{1}_{h(\boldsymbol y_1) > \lambda \land h'(\boldsymbol y_1) < - \lambda} + \sh{h}{h'} \right)
    \end{aligned}
    \end{equation}
    which leads to
    \begin{equation}
    \begin{aligned}
        r = \max_{(\alpha, \beta) \in \mathcal{I}} r_{\alpha, \beta} \leq \sup_{h \in \H, h' \in \H} \left( h(\boldsymbol y_1) - h'(\boldsymbol y_1) - 2 \lambda \mathbbm{1}_{h(\boldsymbol y_1) > \lambda \land h'(\boldsymbol y_1) < - \lambda} + \sh{h}{h'} \right)
    \end{aligned}
    \end{equation}
    proving the bound in (\ref{eq:r helper}).

    Next, we aim to separate the indicator in (\ref{eq:r helper}) from the rest of the expression.
    Let us denote by $h^{\star}$ and $h'^{\star}$ the estimators that achieve the supremum in (\ref{eq:r helper}), which depend on the value of the Rademacher random variables $\epsilon_i,\ i \in [2, m]$.
    Then,
    \begin{equation}
    \begin{aligned}
        \label{eq:T_1_def}
        R & \leq \frac{1}{2} \E_{\{\epsilon_i\}_{i = 2}^m} \left(h^{\star}(\boldsymbol y_1) - h'^{\star}(\boldsymbol y_1) - 2 \lambda \mathbbm{1}_{h^{\star}(\boldsymbol y_1) > \lambda \land h'^{\star}(\boldsymbol y_1) < - \lambda} + \sh{h^{\star}}{h'^{\star}} \right) \\
        & = \frac{1}{2} \E_{\{\epsilon_i\}_{i = 2}^m} \left(h^{\star}(\boldsymbol y_1) - h'^{\star}(\boldsymbol y_1) + \sh{h^{\star}}{h'^{\star}}\right) - \lambda \E_{\{\epsilon_i\}_{i = 2}^m} \mathbbm{1}_{h^{\star}(\boldsymbol y_1) > \lambda \land h'^{\star}(\boldsymbol y_1) < - \lambda}\\
        & = \frac{1}{2} \E_{\{\epsilon_i\}_{i = 2}^m} \left(h^{\star}(\boldsymbol y_1) - h'^{\star}(\boldsymbol y_1) + \sh{h^{\star}}{h'^{\star}} \right) - \lambda T_1
    \end{aligned}
    \end{equation}
    where we denoted $T_1 =  \E_{\{\epsilon_i\}_{i = 2}^m} \mathbbm{1}_{h^{\star}(\boldsymbol y_1) > \lambda \land h'^{\star}(\boldsymbol y_1) < - \lambda} \in [0, 1]$.
    The RC is a nonnegative quantity, which translates to restrictions on the value of $T_1$.
    Note, that the supremum over the first term in (\ref{eq:T_1_def}), increases its value, leading to
    \begin{equation}
    \begin{aligned}
        \label{eq:ST lemma helper 1}
        R & \leq \frac{1}{2} \E_{\{\epsilon_i\}_{i = 2}^m} \sup_{h, h' \in \H} \left(h(\boldsymbol y_1) - h'(\boldsymbol y_1) + \sh{h}{h'}\right) - \lambda T_1.
    \end{aligned}
    \end{equation}
    
    Following the same chain of reasoning, we introduce a Rademacher random variable into (\ref{eq:ST lemma helper 1}) and express the above as
    \begin{equation}
    \begin{aligned}
        \label{eq:T_1}
         R & \leq \E_{\{\epsilon_i\}_{i = 1}^m} \sup_{h \in \H} \left(\epsilon_1 h(\boldsymbol y_1) + \sum_{i = 2}^m \epsilon_i \S_{\lambda}(h(\boldsymbol y_i))\right) - \lambda T_1.
    \end{aligned}
    \end{equation}
    Applying the same chain of thoughts on each sample, results in similar reduction in the RC by quantities $T_i \in [0, 1],\ i \in [2,m]$, which are obtained by repeating the process in (\ref{eq:T_1}) for all $m$ samples, similarly to (\ref{eq:T_1_def}).
    Repeating the process for all $i \in [2, m]$, leads to
    \begin{equation}
    \begin{aligned}
        R & \leq \E_{\{\epsilon_i\}_{i = 1}^m} \sup_{h \in \H} \sum_{i = 1}^m \epsilon_i h(\boldsymbol y_i) - \lambda T
    \end{aligned}
    \end{equation}
    where $T = \sum_{i = 1}^m T_i \ \in [0, m]$.
    Substituting the definition of the RC in (\ref{eq:Rademacher_definition}), results in
    \begin{equation}
    \begin{aligned}
    \label{eq:ub_Rad}
        \Ra{\FH{S_{\lambda}}{\H}} = \frac{R}{m} & \leq \Ra{\H } - \frac{\lambda T}{m}.
    \end{aligned}
    \end{equation}
    The RC is nonnegative, leading to
    \begin{equation}
    \begin{aligned}
        T \leq \min \left\{ \frac{m \Ra{\H }}{\lambda}, m \right\},
    \end{aligned}
    \end{equation}
which completes the proof.
\end{proof}

\subsection{Generalization error bounds}
\label{sec: appendix A2 GE}

In this section we bound the GE of ISTA and ADMM networks.
We start by presenting a few useful lemmas, starting by Talagrand’s contraction lemma.

\begin{lemma}[Talagrand’s contraction lemma \citep{Talgrand1991ContractionLemma}]
    Let $\phi$ be a $G$-Lipschitz function. Then
    \begin{equation}
        \E_{\{\epsilon_i\}_{i = 1}^m} \sup_{h \in \H} \frac{1}{m} \sum_{i = 1}^m \epsilon_i \phi \left( h(\boldsymbol y_i) \right)
        \leq G \E_{\{\epsilon_i\}_{i = 1}^m} \sup_{h \in \H} \frac{1}{m} \sum_{i = 1}^m \epsilon_i h(\boldsymbol y_i).
    \label{eq:ratio_classes_function}
\end{equation}
\label{lemma:Talagrand}
\end{lemma}

\begin{lemma}
\label{lemma: absconv}
For any class of functions $\H$, define $\textit{absconv}(\H) = \{\sum \alpha_i h_i:
h_i \in \H, \sum |\alpha_i| = 1\}$. Then \citep{Bartlett2002Rademacher},
\begin{align}
    \Ra{\H} = \Ra{\textit{absconv}(\H)}.
\end{align}
\end{lemma}

\begin{lemma}
\label{lemma:norm B}
Consider a class of functions $\H$ such that $h: \mathbb{R}^{n_x} \xrightarrow[]{} \mathbb{R},\ h \in \H$.
Denote by $\boldsymbol h(\boldsymbol y)$ the vector of estimators such that any entry results from an estimator in the same class of functions $\boldsymbol h_j \in \H,\ j \in [1, n_x]$. 
If $\boldsymbol w \in \mathbb{R}^{n_x}$ such that $\norm{\boldsymbol w}_{1} < B$, then
\begin{equation}
    \begin{aligned}
        \E_{\{\boldsymbol \epsilon_i\}_{i = 1}^m} \sup_{\boldsymbol w, \boldsymbol h} \frac{1}{m} \sum_{i = 1}^m \epsilon_{i} \boldsymbol w^T \boldsymbol h(\boldsymbol y_i)
        & \leq B \E_{\{\epsilon_i\}_{i = 1}^m} \sup_{h \in \H} \frac{1}{m} \sum_{i = 1}^m \epsilon_{i} h(\boldsymbol y_i).
    \end{aligned}
    \end{equation}
\end{lemma}

\begin{proof}
    Computing the LHS explicitly
    \begin{equation}
    \begin{aligned}
        \label{eq:B_helper1}
        & \E_{\{\epsilon_i\}_{i = 1}^m} \sup_{\norm{\boldsymbol w}_1 \leq B, \boldsymbol h} \frac{1}{m} \sum_{i = 1}^m \epsilon_{i} \boldsymbol w^T \boldsymbol h(\boldsymbol y_i)\\
        & = \E_{\{\epsilon_i\}_{i = 1}^m} \sup_{\norm{\boldsymbol w}_1 \leq B, \boldsymbol h} \norm{\boldsymbol w}_1 \frac{1}{m} \sum_{i = 1}^m \epsilon_{i} \frac{\boldsymbol w^T}{\norm{\boldsymbol w}_1} \boldsymbol h(\boldsymbol y_i)\\
        &\leq B \E_{\{\epsilon_i\}_{i = 1}^m} \sup_{\norm{\boldsymbol w}_1 \leq B, \boldsymbol h} \frac{1}{m} \sum_{i = 1}^m \epsilon_{i}  \frac{\boldsymbol w^T}{\norm{\boldsymbol w}_1} \boldsymbol h(\boldsymbol y_i) \\
        & = B \E_{\{\epsilon_i\}_{i = 1}^m} \sup_{\norm{\boldsymbol{\tilde{w}}}_1 = 1, \boldsymbol h} \frac{1}{m} \sum_{i = 1}^m \epsilon_{i}  \boldsymbol{ \tilde{w}}^T \boldsymbol h(\boldsymbol y_i)
    \end{aligned}
    \end{equation}
    where $\boldsymbol{\tilde{w}} = \boldsymbol{w} / \norm{\boldsymbol{w}}_1$.
    Since $\norm{\boldsymbol{\tilde{w}}}_1 = 1$, applying Lemma \ref{lemma: absconv} results in
    \begin{equation}
    \begin{aligned}
        \label{eq:B_helper2}
        \E_{\{\epsilon_i\}_{i = 1}^m} \sup_{\norm{\boldsymbol{\tilde{w}}}_1 = 1, \boldsymbol h} \frac{1}{m} \sum_{i = 1}^m \epsilon_{i}  \boldsymbol{ \tilde{w}}^T \boldsymbol h(\boldsymbol y_i) 
        & = \E_{\{\epsilon_i\}_{i = 1}^m} \sup_{\norm{\boldsymbol{\tilde{w}}}_1 = 1, \boldsymbol h} \frac{1}{m} \sum_{i = 1}^m \epsilon_{i} \sum_{j = 1}^{n_x} \boldsymbol{ \tilde{w}}_j \boldsymbol h_j(\boldsymbol y_i)
        \\ & = \E_{\{\epsilon_i\}_{i = 1}^m} \sup_{h \in \textit{absconv}(\H)} \frac{1}{m} \sum_{i = 1}^m \epsilon_{i} h(\boldsymbol y_i)\\
        & = \E_{\{\epsilon_i\}_{i = 1}^m} \sup_{h \in \H} \frac{1}{m} \sum_{i = 1}^m \epsilon_{i} h(\boldsymbol y_i),
    \end{aligned}
    \end{equation}
    where the last equality results from Lemma \ref{lemma: absconv}.
    Combining (\ref{eq:B_helper1}) and (\ref{eq:B_helper2}), completes the proof.
\end{proof}

\textbf{Theorem \ref{theorem: GE of ISTA network}} (Generalization error bound for ISTA networks): \\
    Consider the class of learned ISTA networks of depth $L$ as described in (\ref{eq:ISTA_unfolded}), and $m$ i.i.d. samples.
    Then there exist $T^{(l)}$ for $\ l \in [1, L]$ in the range
    \begin{equation}
    \begin{aligned}
        T^{(l)} \in \left[ 0, \min \left\{ \frac{m B_{l} G_I^{l - 1}}{\lambda}, m \right\} \right]
    \end{aligned}
    \end{equation}
    where
    \begin{equation}
    \begin{aligned}
        G_I^l = \frac{B_0 \prod_{l' = 1}^l B_{l'}}{\sqrt{m}} - \frac{\lambda}{m} \sum_{l' = 1}^{l-1} T^{(l')} \prod_{j = l'+1}^{l} B_j  - \frac{\lambda T^{(l)}}{m} ,
    \end{aligned}
    \end{equation}
    satisfying
    \begin{equation}
    \begin{aligned}
        \GE{\left( \HV_I^L \right)} \leq 2 \E_{S} G_I^L.
    \end{aligned}
    \end{equation}
    
\begin{proof}
    To bound the GE of $\HV_I^L$, we rely on the relation with the RC of $\HV_I^L$, described in (\ref{eq:GE bound by Rademacher}).
    Using the assumption that the loss function $\L$ can be written as an average of the per-coordinate losses, leads to
    \begin{equation}
    \begin{aligned}
        \Ra{ \FH{\L}{\HV_I^L}} &= \E_{\{\epsilon_i\}_{i = 1}^m} \sup_{\boldsymbol h^L \in \HV_I^L} \frac{1}{m} \sum_{i = 1}^m \epsilon_{i} \L \left(\boldsymbol h^L(\boldsymbol y_i) , \boldsymbol x_i \right) \\
        &= \E_{\{\epsilon_i\}_{i = 1}^m} \sup_{\boldsymbol h^L \in \HV_I^L} \frac{1}{m} \sum_{i = 1}^m \epsilon_{i} \frac{1}{n_x} \sum_{j = 1}^{n_x} \ell \left( \boldsymbol h_j^L(\boldsymbol y_i) , \boldsymbol x_{i,j} \right)
    \end{aligned}
    \end{equation}
    where $\boldsymbol x_{i,j}$ denotes the $j$th coordinate of the $i$th sample.
    We relax the above supremum, by taking the supremum on each coordinate separately
    \begin{equation}
    \begin{aligned}
        \Ra{ \FH{\L}{\HV_I^L}} & \leq \E_{\{\epsilon_i\}_{i = 1}^m} \sup_{\boldsymbol h_j^L \in \H_I^L, j \in [1, n_x]} \frac{1}{m} \sum_{i = 1}^m \epsilon_{i} \frac{1}{n_x} \sum_{j = 1}^{n_x} \ell \left( \boldsymbol h_j^L(\boldsymbol y_i) , \boldsymbol x_{i,j} \right)
        \label{eq: helper ST1}
    \end{aligned}
    \end{equation}
    where $\H_I^L$ is the class of scalar function that represents a single neuron at the output layer.
    We observe that each neuron in the output layer is given by the same class of functions $\H_I^L$.
    Since the supremum applies on a separable function, (\ref{eq: helper ST1}) reads
    \begin{equation}
    \begin{aligned}
        \Ra{ \FH{\L}{\HV_I^L}} & \leq \frac{1}{n_x} \sum_{j = 1}^{n_x} \E_{\{\epsilon_i\}_{i = 1}^m} \sup_{\boldsymbol h_j^L \in \H_I^L} \frac{1}{m} \sum_{i = 1}^m \epsilon_{i} \ell \left( \boldsymbol h_j^L(\boldsymbol y_i) , \boldsymbol x_{i,j} \right).
        \label{eq: helper ST2}
    \end{aligned}
    \end{equation}
    Applying a $1$-Lipschitz loss (with respect to its first entry) on the network's prediction, satisfies the same bound, as detailed in \citep{Gao2016Dropout} (the proof relies on Lemma \ref{lemma:Talagrand}).
    As a result
    \begin{equation}
    \begin{aligned}
        \Ra{ \FH{\L}{\HV_I^L}} & \leq \frac{1}{n_x} \sum_{j = 1}^{n_x} \E_{\{\epsilon_i\}_{i = 1}^m} \sup_{\boldsymbol h_j^L \in \H_I^L} \frac{1}{m} \sum_{i = 1}^m \epsilon_{i} \boldsymbol h_j^L(\boldsymbol y_i).
        \label{eq: helper ST3}
    \end{aligned}
    \end{equation}
    We observe that the supremum in (\ref{eq: helper ST3}) is repeated for each coordinate, although the function is taken over the same class of scalar functions $\H_I^L$.
    The bound in (\ref{eq: helper ST3}) can be replaced with
    \begin{equation}
    \begin{aligned}
        \Ra{ \FH{\L}{\HV_I^L}} & \leq \E_{\{\epsilon_i\}_{i = 1}^m} \sup_{ h^L \in \H_I^L} \frac{1}{m} \sum_{i = 1}^m \epsilon_{i} h^L(\boldsymbol y_i).
        \label{eq: helper ST4}
    \end{aligned}
    \end{equation}
    For the rest of the proof, we focus on bounding the RC of a single neuron in the ISTA network, denoted by
    \begin{equation}
    \begin{aligned}
        R_I^L & := \E_{\{\epsilon_i\}_{i = 1}^m} \sup_{ h^L \in \H_I^L} \frac{1}{m} \sum_{i = 1}^m \epsilon_{i} h^L(\boldsymbol y_i).
    \end{aligned}
    \end{equation}
    
    Consider the change of the RC between classes of functions representing networks with a  consecutive number of layers.
    The RC of $\H_I^{l + 1}$ is
    \begin{equation}
    \begin{aligned}
        R_I^{l+1} = \Ra{\H_I^{l + 1}} = \E_{\{\epsilon_i\}_{i = 1}^m} \sup_{\boldsymbol w^{l+1}, \boldsymbol h^{l}} \frac{1}{m} \sum_{i = 1}^m \epsilon_{i}  \S_{\lambda}\left(\boldsymbol w^{l+1,T} \boldsymbol h^l(\boldsymbol y_i) + b \right)
    \end{aligned}
    \end{equation}
    where $\boldsymbol w^{l+1}$ is a row in the weights matrix of the corresponding layer (such that $\norm{\boldsymbol w^{l+1}}_1 \leq B_{l+1}$), and $\boldsymbol h^{l}$ is a vector of estimators such that any entry results from an estimator in the class of functions $\H_I^l$, which results from previous layer in the network.
    In addition, $b$ corresponds to a single entry of the bias term $\boldsymbol b$.
    
    Applying the bound on the RC due to the soft-thresholding from Lemma \ref{lemma:ST} leads to
    \begin{equation}
    \begin{aligned}
        \label{eq:GE1_helper1}
        R_I^{l + 1} & \leq \E_{\{\epsilon_i\}_{i = 1}^m} \sup_{\boldsymbol w^{l+1}, \boldsymbol h^{l}} \frac{1}{m} \sum_{i = 1}^m \epsilon_{i} \left(\boldsymbol w^{l+1} \boldsymbol h^{l}(\boldsymbol y_i) + b \right) - \frac{\lambda T^{(l + 1)}}{m} \\
        & = \E_{\{\epsilon_i\}_{i = 1}^m} \sup_{\boldsymbol w^{l+1}, \boldsymbol h^{l}} \frac{1}{m} \sum_{i = 1}^m \epsilon_{i} \boldsymbol w^{l+1} \boldsymbol h^{l}(\boldsymbol y_i) - \frac{\lambda T^{(l + 1)}}{m}
    \end{aligned}
    \end{equation}
    where $T^{(l + 1)}$ results from Lemma \ref{lemma:ST} and corresponds to the $l$th layer in the network.
    The last equality follows from the symmetry of the Rademacher random variables which cancels the dependence on $b$, since it is a constant scalar.
    We observe that the assumption of a constant bias is used in (\ref{eq:GE1_helper1}).
    Since it is the only place where we used this assumption, obtaining the result in (\ref{eq:GE1_helper1}) without the restriction of a constant bias, will allow to relax this assumption.
    
    Applying Lemma \ref{lemma:norm B} implies
    \begin{equation}
    \begin{aligned}
        R_I^{l + 1} & \leq B_{l+1} \E_{\{\epsilon_i\}_{i = 1}^m} \sup_{h^{l}} \frac{1}{m} \sum_{i = 1}^m \epsilon_{i} h^{l}(\boldsymbol y_i) - \frac{\lambda T^{(l + 1)}}{m}.
    \end{aligned}
    \end{equation}
    Substituting into the RC definition from (\ref{eq:Rademacher_definition}), leads to a recurrence relation on the RC of consecutive layers in the network
    \begin{equation}
    \begin{aligned}
        \label{eq: Rademacher recurence relation}
        R_I^{l + 1} \leq B_{l+1} R_I^{l} - \frac{\lambda T^{(l+1)}}{m}.
    \end{aligned}
    \end{equation}
    Moreover, the first layer of the network is obtained by applying a linear mapping, followed by the soft-thresholding operator.
    The bound on the RC of linear predictors, $\H_L = \left\{ \boldsymbol{x} \mapsto \boldsymbol{w} \boldsymbol{x} :\, \norm{\boldsymbol{w}}_2 \leq B_1\right\}$, is given by
    \begin{equation}
    \begin{aligned}
         \label{eq:linear predictor bound}
        R_{\H_L} \leq \frac{B_1 B_0}{\sqrt{m}}
    \end{aligned}
    \end{equation}
    where $B_0$ bounds the network's initialization as detailed in Section \ref{sec:Network_Architecture} \citep{Shalev_Shwartz_book}.
    Therefore, the RC of the first layer is obtained by applying the bound in (\ref{eq:linear predictor bound}), followed by Lemma \ref{lemma:ST}
    \begin{equation}
    \begin{aligned}
        R_I^{1} \leq \frac{B_0 B_1}{\sqrt{m}} - \frac{\lambda T^{(1)}}{m}.
    \end{aligned}
    \end{equation}
    Applying the recurrence relation $L $ times results in
    \begin{equation}
    \begin{aligned}
        R_I^{L} \leq G_I^{L} = \frac{B_0 \prod_{l = 1}^{L} B_l}{\sqrt{m}} - \frac{\lambda}{m} \sum_{l = 1}^{L -1} T^{(l)} \prod_{j = l+1}^{L } B_j - \frac{\lambda T^{(L)}}{m}
        \label{eq:GE helper 1}
    \end{aligned}
    \end{equation}
    where $G_I^L$ is obtained by the same recurrence relation as in (\ref{eq: Rademacher recurence relation}).
    To ensure that the bound on the RC is nonnegative, we restrict the value of $T^{(l+1)}$ to
    \begin{equation}
    \begin{aligned}
        T^{(l+1)}\in \left[ 0, \min \left\{\frac{m B_{l + 1} G_I^l}{\lambda}, m  \right\}\right].
    \end{aligned}
    \end{equation}
    
    
    Following (\ref{eq:GE bound by Rademacher}), the GE is bounded by
    \begin{equation}
    \begin{aligned}
        \GE{\left( \HV_I^L \right)} \leq 2 \E_{S} \Ra{ \FH{\L}{\HV_I^L}} \leq 2 \E_{S} R_I^L \leq 2 \E_{S} G_I^L,
    \end{aligned}
    \end{equation}
which completes the proof.
\end{proof}

Next, we show that for a specific distribution, the expected value of $T$ is greater than $0$, demonstrating that for specific network parameters, ISTA and ADMM networks achieve a significant reduction in the GE.
In the following proposition, the quantities $T^{(l)}$, $B_l$, and $\boldsymbol w'^{l}$, for $l \in [1, L]$, follow from Theorem \ref{theorem: GE of ISTA network}.

\begin{proposition}
Consider the distribution of $T^{(l)}$ where $\boldsymbol b_i$ is a vector of random variables, ranging in $[-B_0, B_0]$.
We assume there exists a constant $c \in [0, 1]$ such that for the underlying data distribution $\D$, $\mathbb{E}_{\D}(e^{-\boldsymbol w^{l,\star, T} \boldsymbol b_1}) \leq e^{-c B_l b^{(l)}}$ and $\mathbb{E}_{\D}(e^{\boldsymbol w'^{l,\star, T} \boldsymbol b_1}) \leq e^{-c B_l b^{(l)}}$, where $b^{(l + 1)} = B_l b^{(l)} - \lambda$, $b^{(1)} = B_0$.
In addition, the weights $\boldsymbol w^{l,\star}$ and $\boldsymbol w'^{l,\star}$ are obtained by the optimal estimators from Lemma \ref{lemma:ST}.
Then the expected value of $T^{(l)}$ is lower bounded by
    \begin{equation}
    \begin{aligned}
        \mathbb{E}_{S} \left( T^{(l)} \right) &\geq \max \left\{ m \left(1 - 2e^{-(c B_l b^{(l)} - \lambda)} \right), 0 \right\}.
    \end{aligned}
    \end{equation}
\end{proposition}

\begin{proof}

From the definition in (\ref{eq:T_1_def}), the expected value of $T^{(1)}_1$ is:
    \begin{equation}
    \begin{aligned}
        \mathbb{E}_{\D} \left( T^{(1)}_1 \right) &=  \mathbb{E}_{\D} \left(\E_{\{\epsilon_i\}_{i = 2}^m} \mathbbm{1}_{h^{1,\star}(\boldsymbol y_1) > \lambda \land h'^{1,\star}(\boldsymbol y_1) < - \lambda}\right) \\
        &=  \mathbb{E}_{\{\epsilon_i\}_{i = 2}^m} \left(\E_{\D} \mathbbm{1}_{h^{1,\star}(\boldsymbol y_1) > \lambda \land h'^{1,\star}(\boldsymbol y_1) < - \lambda}\right) \\
        &= \mathbb{E}_{\{\epsilon_i\}_{i = 2}^m} \left( \mathbb{P}_{\D} \left(h^{1,\star}(\boldsymbol y_1) > \lambda \land h'^{1,\star}(\boldsymbol y_1) < - \lambda \right) \right) \\
        &= \mathbb{E}_{\{\epsilon_i\}_{i = 2}^m} \left( 1 - \mathbb{P}_{\D} \left(h^{1,\star}(\boldsymbol y_1) \leq \lambda \lor h'^{1,\star}(\boldsymbol y_1) \geq - \lambda \right) \right)\\
        &\geq \mathbb{E}_{\{\epsilon_i\}_{i = 2}^m} \left( 1 - \mathbb{P}_{\D} \left(h^{1,\star}(\boldsymbol y_1) \leq \lambda \right) - \mathbb{P}_{\D} \left(h'^{1,\star}(\boldsymbol y_1) \geq - \lambda \right) \right)
    \end{aligned}
    \end{equation}
where we replaced the expected  value of the indicator function, by the probability of the event, and applied the union bound.
We substitute the outcome of the first layer by $h^{1,\star}(\boldsymbol y_1) = \boldsymbol w^{1,\star, T} \boldsymbol b_1$ and $h'^{1,\star}(\boldsymbol y_1) = \boldsymbol w'^{1,\star, T} \boldsymbol b_1$, where $\boldsymbol b_1$ is the bias term defined in Section \ref{sec:Network_Architecture}.

Following the assumption, there exists a constant $c \in [0, 1]$ such that $\mathbb{E}_{\D}(e^{-\boldsymbol w^{1,\star, T} \boldsymbol b_1}) \leq e^{-c B B_0}$.
This assumption captures the relation between the optimal weights $\boldsymbol w^{1,\star}$ and the inputs $\boldsymbol b_1$.
Since the entries of $\boldsymbol w^{1,\star}$ and $\boldsymbol b_1$ are smaller than $B_0$ and $B$ in absolute value, the constant $c$ captures how the product of both vectors is close to their maximal value.

Applying Chernoff bound leads to
    \begin{equation}
    \begin{aligned}
        \mathbb{P}_{\D} \left(h^{1,\star}(\boldsymbol y_1) \leq \lambda \right) &\leq \frac{\mathbb{E}_{\D} \left( e^{- h^{1,\star}(\boldsymbol y_1)} \right)}{e^{-\lambda}} 
        & = \frac{\mathbb{E}_{\D} \left( e^{-\boldsymbol w^{1,\star, T} \boldsymbol b_1} \right)}{e^{-\lambda}} 
        & \leq \frac{\mathbb{E}_{\D} \left( e^{- c B B_0} \right)}{e^{-\lambda}} 
        &= e^{ \lambda - c B B_0}.
    \end{aligned}
    \end{equation}
Similarly, we assume that $\mathbb{E}_{\D}(e^{\boldsymbol w^{1,\star, T} \boldsymbol b_1}) \leq e^{-c B B_0}$, and obtain
    \begin{equation}
    \begin{aligned}
        \mathbb{P}_{\D} \left(h'^{1,\star}(\boldsymbol y_1) \geq -\lambda \right) &\leq \frac{\mathbb{E}_{\D} \left( e^{ h'^{1,\star}(\boldsymbol y_1)} \right)}{e^{-\lambda}} 
        & = \frac{\mathbb{E}_{\D} \left( e^{\boldsymbol w'^{1,\star, T} \boldsymbol b_1} \right)}{e^{-\lambda}} 
        & \leq \frac{\mathbb{E}_{\D} \left( e^{- c B B_0} \right)}{e^{-\lambda}} 
        &= e^{ \lambda - c B B_0}.
    \end{aligned}
    \end{equation}
As a result,
    \begin{equation}
    \begin{aligned}
        \mathbb{E}_{\D} \left( T^{(1)}_1 \right) &\geq \mathbb{E}_{\{\epsilon_i\}_{i = 2}^m} \left( 1 - 2e^{-(c B B_0 - \lambda)} \right) 
        &= 1 - 2e^{-(c B B_0 - \lambda)}.
    \end{aligned}
    \end{equation}

Repeating this process for all samples, leads to the overall reduction obtained by the first layer
    \begin{equation}
    \begin{aligned}
        \mathbb{E}_{S} \left( T^{(1)} \right) &\geq m \left(1 - 2e^{-(\lambda - c B B_0)} \right).
        \label{eq:bound T1}
    \end{aligned}
    \end{equation}
We do not consider the case of $\lambda > c B B_0$, since many of the entries will be zeroed by the soft thresholding operation.
For example, when $c = 1$, only the zero estimator can be achieved, which is not relevant for learning.
To obtain a meaningful bound which is greater than $0$, we assume that $\lambda < c B B_0 + \ln(2)$.

Next, we consider the distribution of $T^{(l)}$ for a general layer $l \in [1, L]$.
Now, the entries of the estimator $\boldsymbol h^{l,\star}$ range in $[-b^{(l)}, b^{(l)}]$, where
    \begin{equation}
    \begin{aligned}
        b^{(l + 1)} = B b^{(l)} - \lambda
    \end{aligned}
    \end{equation}
and $b^{(1)} = B_0$.
Again, we assume that there exists a constant $c \geq 0$ such that $\mathbb{E}_{\D}(e^{-\boldsymbol w^{1,\star, T} \boldsymbol b_1}) \leq e^{-c B b^{(l)}}$.
Following the above derivation and noticing that the bound is non-negative, the expected value of $T^{(l)}$ is bounded by
    \begin{equation}
    \begin{aligned}
        \mathbb{E}_{S} \left( T^{(l)} \right) &\geq \max \left\{ m \left(1 - 2e^{-(c B b^{(l)} - \lambda)} \right), 0 \right\},
        \label{eq:bound Tl}
    \end{aligned}
    \end{equation}
completing the proof.
\end{proof}

To obtain a meaningful bound which is greater than $0$, we assume that $\lambda < c B b^{(l)} + \ln(2)$, which translates to $\lambda \leq \frac{c B^{l + 1} B_0 + \ln(2)}{1 + l c B}$.
The result in (\ref{eq:bound Tl}), behaves as expected with respect to the network's parameters $n_x, \lambda, B$, and $B_0$.
Increasing the value of the soft threshold or decreasing $B$, will decrease the bound in (\ref{eq:bound Tl}), since crossing the threshold is less probable.

Taking the limit of (\ref{eq:GE helper 1}) for $B_l \xrightarrow[]{} 1, l \in [1, L]$, the bound reduces to
\begin{equation}
\begin{aligned}
    \GE{\left( \HV_I^L \right)} \leq 2 \left( \frac{B_0}{\sqrt{m}} - \frac{\lambda}{m} \sum_{l = 1}^{L} \E_S(T^{(l)}) \right)
\end{aligned}
\end{equation}
showing that the GE decreases with the number of layers.
This is in contrast to ReLU networks, where the GE increases exponentially with the network's depth, as shown in Section \ref{sec:GE of ReLU}.
This bound also improves over previously available bounds for ISTA networks, where the increase is logarithmic with the number of layers \citep{Behboodi2020MB_RC,Schnoor2021MB_RC}.

To obtain a more compact relation, we can choose the maximal matrices' norm $B = \max_{l \in [1,L]} B_l$, and denote $T = \min_{l  \in [1,L]} T^{(l)} \in [0, m]$.
The recurrence relation in (\ref{eq: Rademacher recurence relation}) then reads
\begin{equation}
\begin{aligned}
    \label{eq: G_I recurence relation}
    G_I^{l + 1} = B G_I^{l} - \frac{\lambda \E_S(T)}{m}.
\end{aligned}
\end{equation}
This means that 
\begin{equation}
\begin{aligned}
    G_I^{L} = \frac{B_0 B^L}{\sqrt{m}} - \frac{\lambda \E_S(T)}{m} \frac{B^L - 1}{B - 1},
\end{aligned}
\end{equation}
which results in a simpler bound
\begin{equation}
\begin{aligned}
    \GE{\left( \HV_I^L \right)} \leq 2 G_I^{L} = 2 \left( \frac{B_0 B^L}{\sqrt{m}} - \frac{\lambda \E_S(T)}{m} \frac{B^L - 1}{B - 1} \right).
\end{aligned}
\end{equation}

To verify under what condition a nonincreasing GE is obtained, we substract between the bounds of the GE of networks with consecutive number of layers
\begin{equation}
\begin{aligned}
    G_I^{l} - G_I^{l - 1} &= \frac{B_0 B^{l - 1}}{\sqrt{m}} (B - 1) - \frac{\lambda \E_S(T)}{m} \frac{B^{l - 1}}{B - 1} (B - 1) \\
    & = \frac{B^{l - 1}}{\sqrt{m}} \left( B_0 (B - 1) - \frac{\lambda \E_S(T)}{m} \right).
\end{aligned}
\end{equation}
As a result, a nonincreasing GE is achievable by restricting the matrices' norm to satisfy
\begin{equation}
\begin{aligned}
    B \leq 1 + \frac{\lambda \E_S(T)}{\sqrt{m} B_0}.
    \label{eq:ISTA bound B}
\end{aligned}
\end{equation}
We observe that the value of $T$ is also dependent on $B$ (as is seen from (\ref{eq:T_1_def})), and therefore the design rule results in an implicit function.
This result indicates how large an intermediate result of the network can be increased, as a function of $\lambda \E_S(T)$, without increasing the GE.
Moreover, for all combinations of $\lambda, G$ and $m$, there exists a value of $B$ such that $B > 1$.

\textbf{Theorem \ref{theorem: GE of ADMM network}} (Generalization error bound of ADMM networks): \\
    Consider the class of learned ADMM networks of depth $L$ as described in (\ref{eq:ADMM_unfolded}), and $m$ i.i.d. samples.
    Then there exist $T^{(l)}$ for $l \in [1, L-1]$ in the interval
    \begin{equation}
    \begin{aligned}
        T^{(l)}\in \left[ 0, \min \left\{\frac{m \tilde{B}_{l} G_A^{l - 1}}{\tilde{\lambda}}, m  \right\}\right].
    \end{aligned}
    \end{equation}
    where
    \begin{equation}
    \begin{aligned}
        G_A^{l} = \frac{B_0 \prod_{l' = 1}^{l-1} \tilde{B_{l'}}}{\sqrt{m}} - \frac{\tilde{\lambda}}{m} \sum_{l' = 1}^{l-2} T^{(l')} \prod_{j = l'+1}^{l - 1} \tilde{B}_j - \frac{\tilde{\lambda} T^{(l-1)}}{m}.
    \end{aligned}
    \end{equation}
    where $\tilde{\lambda} = (1 + \gamma) \lambda$ and $\tilde{B}_l = (1 + 2 \gamma) (B_l + 2), \ l \in [1,L]$, satisfying
    \begin{equation}
    \begin{aligned}
        \GE\left( \H_A^L \right)\leq 2 \tilde{B}_L \E_{S} G_A^{L - 1}.
    \end{aligned}
    \end{equation}
    
\begin{proof}
    The ADMM recurrence relation in (\ref{eq:ADMM_unfolded}) can be re-written as
    \begin{equation}
    \begin{aligned}
        \boldsymbol z^l & = \S_{\lambda} \left( \boldsymbol b + \boldsymbol W^L \boldsymbol z^{l - 1} + (\boldsymbol W^L - \boldsymbol I) \boldsymbol u^{l - 1} \right)  \\
        \boldsymbol u^l & = \left( \boldsymbol I - \gamma \boldsymbol W^L \right) \boldsymbol u^{l - 1} - \gamma \boldsymbol W^L \boldsymbol z^{l - 1} + \gamma \boldsymbol z^{l}.
    \end{aligned}
    \end{equation}
    
    Following the same logic from the proof of Theorem \ref{theorem: GE of ISTA network}, we focus on a single neuron in the layers.
    We denote $R_z^l = \E_{\{\epsilon_i\}_{i = 1}^m} \sup_{z^{l}} \frac{1}{m} \sum_{i = 1}^m \epsilon_{i} z^{l - 1}$ and $R_u^l = \E_{\{\epsilon_i\}_{i = 1}^m} \sup_{u^{l}} \frac{1}{m} \sum_{i = 1}^m \epsilon_{i} u^{l}$, the RC of a single entry in the $\boldsymbol z^l$ and $\boldsymbol u^l$, respectively.
    Then
    \begin{equation}
    \begin{aligned}
        R_z^{l} & = \E_{\{\epsilon_i\}_{i = 1}^m} \sup_{\boldsymbol w^{l}, \boldsymbol z^{l-1}, \boldsymbol u^{l-1}} \frac{1}{m} \sum_{i = 1}^m \epsilon_{i}  \S_{\lambda} \left( \boldsymbol w^l \boldsymbol z^{l - 1} + (\boldsymbol w^l - \boldsymbol I) \boldsymbol u^{l - 1} +\boldsymbol b \right).
    \end{aligned}
    \end{equation}
    From Lemma \ref{lemma:ST}, the soft-thresholding leads to
    \begin{equation}
    \begin{aligned}      
        R_z^{l}& \leq \E_{\{\epsilon_i\}_{i = 1}^m} \sup_{\boldsymbol w^{l}, \boldsymbol z^{l-1}, \boldsymbol u^{l-1}} \frac{1}{m} \sum_{i = 1}^m \epsilon_{i}  \left( \boldsymbol w^l \boldsymbol z^{l - 1} + (\boldsymbol w^l - \boldsymbol I) \boldsymbol u^{l - 1} \right) - \frac{\lambda T^{(l)}}{m}.
    \end{aligned}
    \end{equation}
    By splitting the supremum and applying Lemma \ref{lemma:norm B}, the above reads
    \begin{equation}
    \begin{aligned}
        R_z^{l} & \leq B_l \E_{\{\epsilon_i\}_{i = 1}^m} \sup_{z^{l-1}} \frac{1}{m} \sum_{i = 1}^m \epsilon_{i} z^{l - 1} + (1 + B_l) \E_{\{\epsilon_i\}_{i = 1}^m} \sup_{u^{l-1}} \frac{1}{m} \sum_{i = 1}^m \epsilon_{i} u^{l - 1} - \frac{\lambda T^{(l)}}{m} \\
        &\leq B_l R_z^{l - 1} + (1 + B_l) R_u^{l - 1} - \frac{\lambda T^{(l)}}{m}.
    \end{aligned}
    \end{equation}
    Similarly for $R_u^{l}$ we get,
    \begin{equation}
    \begin{aligned}
        R_u^{l} & \leq (1 + \gamma B_l) R_u^{l - 1} + \gamma B_l R_z^{l - 1} + \gamma R_z^{l} \\
        & \leq \left( 1 + \gamma (2B_l + 1) \right) R_u^{l - 1} + 2 \gamma B_l R_z^{l - 1} - \frac{\gamma \lambda T^{(l)}}{m}.
    \end{aligned}
    \end{equation}
    
    Applying Lemma \ref{lemma:ST}, and using the fact that the RC of a sum can be bounded by the sum of the individual RCs \citep{Bartlett2002Rademacher}, results in
    \begin{equation}
    \begin{aligned}
        R_z^{l} & \leq B_l R_z^{l - 1} + (1 + B_l) R_u^{l - 1} - \frac{\lambda T^{(l)}}{m} \\
        R_u^{l} & \leq (1 + \gamma B_l) R_u^{l - 1} + \gamma B_l R_z^{l - 1} + \gamma R_z^{l} \\
        & \leq \left( 1 + \gamma (2B_l + 1) \right) R_u^{l - 1} + 2 \gamma B_l R_z^{l - 1} - \frac{\gamma \lambda T^{(l)}}{m}.
    \end{aligned}
    \end{equation}
    As a result,
    \begin{equation}
    \begin{aligned}
        R_z^{l} + R_u^{l}
        & \leq B_l R_z^{l - 1} + (1 + B_l) R_u^{l - 1} - \frac{\lambda T^{(l)}}{m} + \left( 1 + \gamma (2B_l + 1) \right) R_u^{l - 1} + 2 \gamma B_l R_z^{l - 1} - \frac{\gamma \lambda T^{(l)}}{m} \\
        & = \left( 1 + 2 \gamma \right) B_l R_z^{l - 1} + \left( \left( 1 + 2 \gamma \right) B_l + 2 + \gamma\right) R_u^{l - 1} - \frac{(1 + \gamma) \lambda T^{(l)}}{m} \\
        & \leq \left( 1 + 2 \gamma \right) B_l R_z^{l - 1} + \left( 1 + 2 \gamma \right) (B_l + 2) R_u^{l - 1} - \frac{(1 + \gamma) \lambda T^{(l)}}{m} \\
        & \leq \left( 1 + 2 \gamma \right) (B_l + 2) \left( R_z^{l - 1} + R_u^{l - 1} \right) - \frac{(1 + \gamma) \lambda T^{(l)}}{m}.
    \end{aligned}
    \end{equation}
    
    We obtain a recurrence relation on the sum of RCs $R_z^{l} + R_u^{l}$.
    Repeating the proof of Theorem \ref{theorem: GE of ISTA network} for ISTA networks, and replacing $\lambda$ by $\tilde{\lambda} = (1 + \gamma) \lambda$ and $\tilde{B}_l = (1 + 2 \gamma) (B_l + 2), \ \forall l \in [1,L]$, respectively, leads to
    \begin{equation}
    \begin{aligned}
        R_z^{L - 1} + R_u^{L - 1} \leq G_A^{L - 1} = \frac{B_0 \prod_{l = 1}^{L - 1} \tilde{B}_l}{\sqrt{m}} - \frac{\tilde{\lambda}}{m} \sum_{l = 1}^{L - 2} T^{(l)} \prod_{j = l+1}^{L - 1} \tilde{B}_j - \frac{\tilde{\lambda} T^{(L-1)}}{m}.
    \end{aligned}
    \end{equation}
    The value of $T^{(l)}$ are in the interval
    \begin{equation}
    \begin{aligned}
        T^{(l)}\in \left[ 0, \min \left\{\frac{m \tilde{B}_{l} G_A^{l - 1}}{\tilde{\lambda}}, m  \right\}\right].
    \end{aligned}
    \end{equation}
    Moreover, from the ADMM recurrence relation in (\ref{eq:ADMM_unfolded}), the following holds
    \begin{equation}
    \begin{aligned}
        R_A^{l} \leq  \tilde{B}_{l} \left( R_z^{l - 1} + R_u^{l - 1} \right) 
    \end{aligned}
    \end{equation}
    where $R_A^{l} := \Ra{\H_A^l}$ is the RC of $\H_A^l$, leading to
    \begin{equation}
    \begin{aligned}
        R_A^{L} \leq \tilde{B}_{L} \left(\frac{B_0 \prod_{l = 1}^{L - 1} \tilde{B}_l}{\sqrt{m}} - \frac{\tilde{\lambda}}{m} \sum_{l = 1}^{L - 2} T^{(l)} \prod_{j = l+1}^{L - 1} \tilde{B}_j - \frac{\tilde{\lambda} T^{(L-1)}}{m} \right).
    \end{aligned}
    \end{equation}
    Similarly to the proof of Theorem \ref{theorem: GE of ISTA network}, the bound on the GE is obtained by applying a $1$-Lipschitz loss on the network’s prediction, which concludes the proof.
\end{proof}

Similarly to Theorem \ref{theorem: GE of ISTA network}, Theorem \ref{theorem: GE of ADMM network} results in design rules for ADMM networks with low GE.
We observe that the RC bound of the ADMM network is obtained from the bound of the ISTA network, by replacing $\lambda$ and $B_l$ with $\tilde{\lambda}$ and $\tilde{B}_l,\ l \in [1, L]$.
The relation between $B_l$ and $\tilde{B}_l$ sheds light on the relation between the GE of learned ISTA and ADMM.
Since $\tilde{B}_l > B_l$, the GE bound on ISTA networks is potentially lower compared to ADMM networks with the same weight's norm $B_l$, indicating that ISTA networks have better generalization abilities compared to ADMM networks.
Depending on the behaviour of $T$, as the number of training samples $m$ increases, the difference between the bounds on the model-based and ReLU networks (presented in Theorem \ref{theorem: GE of ReLU network}) might be less significant.
In this case, the soft-thresholding nonlinearity is most valuable in the case of small number of training samples.

Similarly to the GE bound of ISTA networks, we define $\tilde{B} = \max_{l \in [1,L]} \tilde{B}_l$ and extract a more compact bound on the GE of ADMM networks
\begin{equation}
\begin{aligned}
    \GE\left( \H_A^L \right) \leq 2 \tilde{B} \left( \frac{B_0 \tilde{B}^{L-1}}{\sqrt{m}} - \frac{\tilde{\lambda} \E_S(T)}{m} \frac{\tilde{B}^{L - 1} - 1}{\tilde{B} - 1}\right).
    \label{eq:GE ADMM simplified}
\end{aligned}
\end{equation}

\subsection{Estimation error bounds}

Here we prove the EE bounds for ISTA, ADMM, and ReLU networks provided in Theorem \ref{theorem: EE of ISTA network}, derived with the LRC machinery, presented in the following theorem.

\begin{defn}
    A function $\psi: [0, \infty) \to [0, \infty)$ is sub-root if it is nonnegative, nondecreasing, and if $r \mapsto \psi(r) / \sqrt{r}\,$ is nonincreasing for $r>0$.
\label{def:subroot}
\end{defn}

\begin{theorem} (Local Rademacher complexity bound on the estimation error for vector estimators \citep{Yousefi2018LRCMTL})
\label{theorem:LRC MTL}
Let $\HV$ be a class of functions where each coordinate ranges in $[-1, 1]$ and let $\L$ be a loss function satisfying:
\begin{itemize}
    \item There exists an estimator $\boldsymbol h^* \in \HV$ satisfying $\E_{\D} \L \left (\boldsymbol h^* \right) = \inf_{\boldsymbol h \in \HV} \E_{\D} \L (\boldsymbol h )$.
    
    \item The loss $\L$ is an averaged of $1$-Lipschitz per-coordinate losses $\ell$
    \begin{equation}
    \begin{aligned}
        \L( \boldsymbol h) = \frac{1}{n_x} \sum_{j = 1}^{n_x} \ell (\boldsymbol h_j).
    \end{aligned}
    \end{equation}
    
    \item There exists a constant $C \geq 1$ such that for every probability distribution $\D$, and estimator $\boldsymbol h \in \HV$, such that
    \begin{equation}
    \begin{aligned}
        \E_{\D} \sum_{j = 1}^{n_x} (\boldsymbol h_j - \boldsymbol h^*_j )^2 \leq C \E_{\D} \sum_{j = 1}^{n_x} \left( \ell(\boldsymbol h_j) - \ell (\boldsymbol h^*_j ) \right).
    \end{aligned}
    \end{equation}

\end{itemize}
Let $\psi$ be a sub-root function with fixed point $r^*$ such that $\psi(r) \geq \mathfrak{R}_m^{\textit{vec}}(\HV_r), \forall r \geq r^*$ where
\begin{equation}
\begin{aligned}
    \mathfrak{R}_m^{\textit{vec}}(\HV_r) &:= \E_{\{\epsilon_{i,j}\}_{i=1, j=1}^{m, n_x}} \sup_{\boldsymbol h \in \HV_{r}} \frac{1}{m}\sum_{i = 1}^m \frac{1}{n_x}\sum_{j = 1}^{n_x} \epsilon_{i,j} \boldsymbol h_j (\boldsymbol y_i)
\end{aligned}
\end{equation}
where $\HV_r$ is defined in (\ref{def:Ar}).

Then for any $s > 0$, any $K > 1$, and any $r \geq \psi(r)$ with probability at least $1 - e^{-s}$
\begin{equation}
\begin{aligned}
    \EE(\HV) \leq 40 K r^* + \frac{16 C^2 K + 48 C}{m n_x} s.
    \label{eq:EE result LRC}
\end{aligned}
\end{equation}

\end{theorem}

In the case of proper learning (such that the target vector belongs to the class of estimators $\H$, implying that $\inf_{h \in \H} L_{\D} (h) = 0$), the requirements on the loss are satisfied for losses given by the $L_p$ norm ($p \geq 1$), as discussed in \citep{Mendelson2002SampleComplexity}.
The value of $C$ in Theorem \ref{theorem:LRC MTL}, depends on the $L_p$ norm being used as a loss.

We derive the following theorem on the EE of ISTA, ADMM, and ReLU networks.

\textbf{Theorem \ref{theorem: EE of ISTA network}} (Estimation error bound of ISTA, ADMM, and ReLU  networks):
    Consider the class of functions represented by depth-$L$ ISTA networks $\HV_I^L$ as detailed in Section \ref{sec:Network_Architecture}, $m$ training samples, and a loss satisfying Assumption \ref{assumption: loss} with a constant $C$.
    If $|| \boldsymbol W^{l} -\boldsymbol W^{l,*} ||_{\infty} \leq \alpha \sqrt{r}$ for some $\alpha > 0$. Moreover, $B \geq \max\{\alpha \sqrt{r}, 1\}$. Then there exists $T$ in the interval
    \begin{equation}
    \begin{aligned}
        T \in \left[ 0, \min \left\{ \frac{\sqrt{m} B_0 B^{L - 1} 2^L }{\lambda \eta}, m\right\} \right]
    \end{aligned}
    \end{equation}
    where $\eta = \frac{L B^{L-1} (B-1) - B^L + 1}{(B-1)^2}$, such that for any $s > 0$ with probability at least $1 - e^{-s}$
    \begin{equation}
    \begin{aligned}
        \EE{\left( \HV_I^L \right)} \leq 41 r^* + \frac{17 C^2 + 48 C}{m n_x} s
    \end{aligned}
    \end{equation}
    where
    \begin{equation}
    \begin{aligned}
        r^* &= C^2 \alpha^2 \left(\frac{B_0 B^{L-1}2^L}{\sqrt{m}} - \frac{\lambda T}{m} \eta \right)^2.
    \end{aligned}
    \end{equation}
    For the class of functions represented by depth-$L$ ADMM networks $\HV_A^L$, the bound is satisfied with
    \begin{equation}
    \begin{aligned}
        r^* &= C^2 \alpha^2 \left(\frac{B_0  \tilde{B}^{L-2}2^{L-1}}{\sqrt{m}} - \frac{\tilde{\lambda} T}{m} \tilde{\eta} \right)^2
    \end{aligned}
    \end{equation}
    where $\tilde{\lambda} = (1 + \gamma) \lambda$, $\tilde{B} = (1 + 2 \gamma) (B + 2)$, and $\tilde{\eta} = \frac{(L-1) \tilde{B}^{L-2} (\tilde{B}-1) - \tilde{B}^{L-1} + 1}{(\tilde{B}-1)^2}$.
    The bound is also satisfied for the class of functions represented by depth-$L$ ReLU networks $\HV_R^L$, with
    \begin{equation}
    \begin{aligned}
        r^* &=  C^2 \alpha^2 \, \left(\frac{B_0 B^{L-1}2^L}{\sqrt{m}}\right)^2.
    \end{aligned}
    \end{equation}

\begin{proof}

To derive the upper bound on the EE, we will use the LRC framework developed in \citep{Bartlett2005LRC}, and rely on Theorem \ref{theorem:LRC MTL} derived in \citep{Yousefi2018LRCMTL}, assuming that the loss function satisfies Assumption \ref{assumption: loss}.

We define $\H_{I,r}^{l,j}$ to be the class of scalar functions that represent the $j$th neuron (coordinate) at the $l$th output layer
\begin{equation}
\begin{aligned}
    \H_{I,r}^{l,j} = \left\{ \boldsymbol h_j: \boldsymbol h \in \HV_{I,r}^l \right\}.
    \label{def:Ar_scalar}
\end{aligned}
\end{equation}

In order to apply the above theorem, one needs to bound the RC of the class of functions $\H_{I,r}^{l,j}$, defined in (\ref{def:Ar_scalar}).
The weight matrices of $\boldsymbol h^*$ and $\boldsymbol h \in \HV_{I,r}^{l}$ are denoted by $\{ \boldsymbol W^{l, *} \}_{l = 1}^{L }$ and $\{ \boldsymbol W^{l}\}_{l = 1}^{L }$, respectively.
The weights differences is defined as $\boldsymbol {\Delta W}^l = \boldsymbol W^{l} - \boldsymbol W^{l, *}$, satisfying $||\boldsymbol {\Delta W}^l||_{\infty} \leq \alpha \sqrt{r} := \Delta B$.
We can now bound the RC of $\H_{I,r}^{l,j}$:
\begin{equation}
\begin{aligned}
    \Ra{\H_{I,r}^{l,j}} &= \E_{\{\epsilon_i\}_{i = 1}^m} \sup_{h^{l}} \frac{1}{m} \sum_{i = 1}^m \epsilon_{i} h^{l}(\boldsymbol y_i) \\
    &= \E_{\{\epsilon_i\}_{i = 1}^m} \sup_{\boldsymbol W^{l}, h^{l-1}} \frac{1}{m} \sum_{i = 1}^m \epsilon_{i} \S_{\lambda} \left(\boldsymbol W^{l } \boldsymbol h^{l - 1}(\boldsymbol y_i) + \boldsymbol b \right).
\end{aligned}
\end{equation}
Applying Lemma \ref{lemma:ST} results in
\begin{equation}
\begin{aligned}
    \Ra{\H_{I,r}^{l,j}} & \leq \E_{\{\epsilon_i\}_{i = 1}^m} \sup_{\boldsymbol W^{l}, h^{l - 1}} \frac{1}{m} \sum_{i = 1}^m \epsilon_{i} \boldsymbol W^{l} \boldsymbol h^{l - 1}(\boldsymbol y_i) - \frac{\lambda T^{(l)}}{m} \\
    & = \E_{\{\epsilon_i\}_{i = 1}^m} \sup_{\boldsymbol{\Delta W}^{l}, h^{l - 1}} \frac{1}{m} \sum_{i = 1}^m \epsilon_{i} \left( \boldsymbol W^{l, *} + \boldsymbol {\Delta W}^{l } \right) \boldsymbol h^{l - 1}(\boldsymbol y_i) - \frac{\lambda T^{(l)}}{m} \\
    & \leq \left( B + \Delta B \right) \Ra{\H_r^{l - 1}}  - \frac{\lambda T^{(l)}}{m}
\end{aligned}
\end{equation}
where the last inequality holds, by applying Lemma \ref{lemma:norm B} and taking into account that $|| \boldsymbol W^{l, *} + \boldsymbol {\Delta W}^{l}  ||_{\infty} \leq || \boldsymbol W^{l, *}||_{\infty} + ||\boldsymbol {\Delta W}^{l }  ||_{\infty} \leq B + \Delta B$.
Accumulating the contributions from all $L$ layers, we obtain
\begin{equation}
\begin{aligned}
    \Ra{\H_{I,r}^{L,j}} &\leq \frac{B_0 \left( \left(B + \Delta B \right)^L - B^L\right)}{\sqrt{m}} - \frac{\lambda T}{m} \frac{(B+\Delta B)^L - 1}{(B+\Delta B) - 1},
    \label{eq: helper EE 1}
\end{aligned}
\end{equation}
where $T = \min_{l \in [1, L]} T^{(l)}$. In addition, the subtraction of the factor $B^L$ results from the following term $\boldsymbol W^{L,*} \boldsymbol W^{L - 1,*} \boldsymbol W^{L - 2,*} \ldots \boldsymbol W^{1,*} \boldsymbol x_i$, which does not contribute to the complexity, since these are fixed matrices.

Next, we bound the expression in (\ref{eq: helper EE 1}) with a term linearly dependent on $\Delta B$.
Applying Newton's binomial formula, reads
\begin{equation}
\begin{aligned}
    \Ra{\H_{I,r}^{L,j}} &\leq \frac{B_0 \left( \sum_{k=0}^{L}{L\choose k}B^k \Delta B^{L-k} - B^L\right)}{\sqrt{m}} - \frac{\lambda T}{m} \frac{(B+\Delta B)^L - 1}{(B+\Delta B) - 1}\\
    &= \frac{B_0 \Delta B \left( \sum_{k=0}^{L-1}{L\choose k}B^k \Delta B^{L-k-1}\right)}{\sqrt{m}} - \frac{\lambda T}{m} \frac{(B+\Delta B)^L - 1}{(B+\Delta B) - 1}.
\end{aligned}
\end{equation}
Assuming that $\Delta B < B$ the above is bounded by
\begin{equation}
\begin{aligned}
    \Ra{\H_{I,r}^{L,j}} & \leq \frac{B_0 \Delta B \left( \sum_{k=0}^{L-1}{L\choose k}B^{L-1} \right)}{\sqrt{m}} - \frac{\lambda T}{m} \frac{(B+\Delta B)^L - 1}{(B+\Delta B) - 1}.
\end{aligned}
\end{equation}
Using the known relation $\sum_{k=0}^{L}{L\choose k} = 2^L$, leads to
\begin{equation}
\begin{aligned}
    \Ra{\H_{I,r}^{L,j}} & \leq \frac{B_0 \Delta B B^{L-1}2^L}{\sqrt{m}} - \frac{\lambda T}{m} \sum_{l=0}^{L-1} (B + \Delta B)^l\\
    & \leq \frac{B_0 \Delta B B^{L-1}2^L}{\sqrt{m}} - \frac{\lambda T}{m} \sum_{l=1}^{L-1} (B + \Delta B)^l \\
    & \leq \frac{B_0 \Delta B B^{L-1}2^L}{\sqrt{m}} - \frac{\lambda T}{m} \sum_{l=1}^{L-1} \sum_{k=0}^{l} {l\choose k}\Delta B^k B^{l-k}
\end{aligned}
\end{equation}
where we used Newton's binomial formula.
Neglecting all terms with $k \neq 1$, reduces the above to
\begin{equation}
\begin{aligned}
    \Ra{\H_{I,r}^{L,j}} & \leq \Delta B \left( \frac{B_0  B^{L-1}2^L}{\sqrt{m}} - \frac{\lambda T}{m} \sum_{l=1}^{L-1} {l\choose 1} B^{l-1} \right)\\
    & = \Delta B \left( \frac{B_0  B^{L-1}2^L}{\sqrt{m}} - \frac{\lambda T}{m} \frac{L B^{L-1} (B-1) - B^L + 1}{(B-1)^2} \right).
\end{aligned}
\end{equation}
Substituting $\Delta B = \alpha \sqrt{r}$, we have
\begin{equation}
\begin{aligned}
    \Ra{\H_{I,r}^{L,j}} & \leq \alpha \sqrt{r} \left( \frac{B_0  B^{L-1}2^L}{\sqrt{m}} - \frac{\lambda T}{m} \frac{L B^{L-1} (B-1) - B^L + 1}{(B-1)^2} \right).
\end{aligned}
\end{equation}
Since the RC is nonnegative, the value of $T$ is restricted to the interval
\begin{equation}
\begin{aligned}
    T \in \left[ 0, \min \left\{ \frac{\sqrt{m} B_0 B^{L - 1} 2^L }{\lambda \eta}, m\right\} \right]
\end{aligned}
\end{equation}
where $\eta = \frac{L B^{L-1} (B-1) - B^L + 1}{(B-1)^2}$.

To simplify our notation, let us denote 
\begin{equation}
\begin{aligned}
    \beta := \left(\frac{B_0  B^{L-1}2^L}{\sqrt{m}} - \frac{\lambda T}{m} \eta \right).
    \label{eq:beta}
\end{aligned}
\end{equation}
Then, for $j \in [1, n_x]$
\begin{equation}
\begin{aligned}
    \Ra{\H_{I,r}^{L, j}} & \leq \alpha \sqrt{r} \beta.
\end{aligned}
\end{equation}

To apply Theorem \ref{theorem:LRC MTL}, we observe that from Jensen's inequality $\mathfrak{R}_m^{\textit{vec}}(\HV_r)$ satisfies
\begin{equation}
\begin{aligned}
    \mathfrak{R}_m^{\textit{vec}}(\HV_r) &= \E_{\{\epsilon_{i,j}\}_{i=1, j=1}^{m, n_x}} \sup_{\boldsymbol h \in \HV_{I,r}^L} \frac{1}{m}\sum_{i = 1}^m \frac{1}{n_x}\sum_{j = 1}^{n_x} \epsilon_{i,j} \boldsymbol h_j (\boldsymbol y_i) \\
    & \leq \frac{1}{n_x}\sum_{j = 1}^{n_x} \E_{\{\epsilon_{i,j}\}_{i=1, j=1}^{m, n_x}} \sup_{\boldsymbol h \in \HV_{I,r}^L} \frac{1}{m}\sum_{i = 1}^m \epsilon_{i,j} \boldsymbol h_j (\boldsymbol y_i) \\
    & = \frac{1}{n_x}\sum_{j = 1}^{n_x} \Ra{\H_{I,r}^{L,j}} \\
    & \leq \frac{1}{n_x}\sum_{j = 1}^{n_x} \alpha \sqrt{r} \beta \\
    & = \alpha \sqrt{r} \beta.
\end{aligned}
\end{equation}

This leads to the following sub-root function
\begin{equation}
\begin{aligned}
    \psi(r) &:= C \alpha \sqrt{r} \beta \geq C \mathfrak{R}_m^{\textit{vec}}(\HV_r).
    \label{eq:subrootpsi}
\end{aligned}
\end{equation}

Since $C,\,\beta>0$ we get that $\psi$ is a nonnegative and nondecreasing function over $r\geq0$.
In addition $\psi(r) / \sqrt{r} = C \alpha \beta $ which is a constant function and in particular is nonincreasing, implying that the conditions for Theorem \ref{theorem:LRC MTL} are satisfied.

The fixed-point of the function, such that $r^* = \psi(r^*)$, can be computed explicitly from
\begin{equation}
\begin{aligned}
    r^* &=C \alpha \sqrt{r^*}  \left(\frac{B_0  B^{L-1}2^L}{\sqrt{m}} - \frac{\lambda T}{m} \eta \right),
\end{aligned}
\end{equation}
which reads
\begin{equation}
\begin{aligned}
    r^* &= C^2 \alpha^2 \left(\frac{B_0  B^{L-1}2^L}{\sqrt{m}} - \frac{\lambda T}{m} \eta \right)^2.
\end{aligned}
\end{equation}
By applying Theorem \ref{theorem:LRC MTL}, we conclude that for any $r \geq r^*$, $K > 1$, and $s > 0$, with probability at least $1 - e^{-s}$
\begin{equation}
\begin{aligned}
    \EE(\HV_{I}^L) \leq 40 K r^* + \frac{16 K C^2 + 48 C}{m n_x} s.
    \label{eq:EE helper 2}
\end{aligned}
\end{equation}
Moreover, we choose $K = 41/40 > 1$, leading to
\begin{equation}
\begin{aligned}
    \EE(\HV_{I}^L) &\leq 41 r^* + \frac{\frac{16 * 41}{40} C^2 + 48 C}{m n_x} s \\
    &\leq 41 r^* + \frac{17 C^2 + 48 C}{m n_x} s.
    \label{eq:EE helper 3}
\end{aligned}
\end{equation}


    Following the observation on the RC of ADMM networks, the RC bounds on ISTA networks are satisfied for ADMM networks by replacing $\lambda$, $B$ by $\tilde{\lambda}$, $\tilde{B}$, respectively, as defined in Theorem \ref{theorem: GE of ADMM network}.
    
    Applying the same proof methodology as Theorem \ref{theorem: EE of ISTA network}, results in the following sub-root function
    \begin{equation}
    \begin{aligned}
        \psi_A(r) &:= C \alpha \sqrt{r} \, 
        \left( \frac{B_0  \tilde{B}^{L-2}2^{L-1}}{\sqrt{m}} - \frac{\tilde{\lambda} T}{m} \frac{(L-1) \tilde{B}^{L-2} (\tilde{B}-1) - \tilde{B}^{L-1} + 1}{(\tilde{B}-1)^2} \right),
        \label{eq:subrootpsi_admm}
    \end{aligned}
    \end{equation}
    leading to the fixed-point
    \begin{equation}
    \begin{aligned}
        r^* &= C^2 \alpha^2 \left(\frac{B_0  \tilde{B}^{L-2}2^{L-1}}{\sqrt{m}} - \frac{\tilde{\lambda} T}{m} \frac{(L-1) \tilde{B}^{L-2} (\tilde{B}-1) - \tilde{B}^{L-1} + 1}{(\tilde{B}-1)^2}\right)^2.
    \end{aligned}
    \end{equation}
    Applying Theorem \ref{theorem:LRC MTL} proves the bound.

    Finally, the same proof methodology applies for ReLU networks, with
    the sub-root function given by
    \begin{equation}
    \begin{aligned}
        \psi_R(r) & = 
        C \alpha \sqrt{r} \, \frac{B_0  B^{L-1}2^L}{\sqrt{m}}
        \label{eq:subrootpsi_relu}
    \end{aligned}
    \end{equation}
    leading to the fixed-point
    \begin{equation}
    \begin{aligned}
        r^* &= C^2 \alpha^2 \, \left(\frac{B_0 B^{L-1}2^L}{\sqrt{m}}\right)^2.
    \end{aligned}
    \end{equation}
    Applying Theorem \ref{theorem:LRC MTL} concludes the proof.
\end{proof}

In contrast to the GE bounds, the role of the weight's norm $B$ is less dominant in the EE bounds presented in Theorem \ref{theorem: EE of ISTA network}. 
Therefore, it is not clear from the EE bounds which model-based networks exhibit better generalization.

The parameter $B_0$ relates the bound to the sparsity level of the target vectors
\begin{equation}
\begin{aligned}
    B_0 = \norm{\boldsymbol b}_1 = \norm{\boldsymbol A^T \boldsymbol y}_1 = \norm{\boldsymbol A^T \boldsymbol A \boldsymbol x + \boldsymbol A^T \boldsymbol e}_1 \leq \norm{\boldsymbol A^T \boldsymbol A \boldsymbol x}_1 + \norm{\boldsymbol A^T \boldsymbol e}_1.
\end{aligned}
\end{equation}
The target vector $\boldsymbol x$ is sparse with sparsity rate of $\rho$, meaning that $\norm{\boldsymbol x}_1 \leq \rho n_x$, so that
\begin{equation}
\begin{aligned}
    B_0 \leq \norm{\boldsymbol A^T \boldsymbol A}_1 \rho n_x + \norm{\boldsymbol A^T \boldsymbol e}_1.
\end{aligned}
\end{equation}
Therefore, as the target vector is more sparse, $\rho$ decreases, which decreases $r^*$, and implies a lower bound on the EE.

\section{Additional simulation results}
\label{sec:Additional simulation results}

In this section, we provide additional results for the EE of ISTA and ReLU networks with varying number of layers.
We show that the EE and performance of the networks behave similarly to the depth-10 networks presented in Section \ref{sec:Simulations}.

In our work, we considered only one set of learned weight matrices. 
However, a second set of weight matrices can also be learned, as understood from the following relation between consecutive layers
$$ \boldsymbol h_I^{l} = \S_{\lambda} \left( \boldsymbol W_1^{l } \boldsymbol h_I^{l - 1} + \boldsymbol W_2^{l} \boldsymbol y\right), \ \ \ \boldsymbol h_I^0 = \S_{\lambda} (\boldsymbol y).$$
The provided comparisons show that the soft thresholding nonlinearity affects the EE of ISTA networks more significantly, compared to learned bias terms, obtained with the set of additional learned matrices.
This observation emphasizes that the performed analysis on the ISTA network with constant biases, could be applicable to additional variations of the ISTA networks.

\begin{figure}[!h]
        \centering
        \includegraphics[width = 4.3in]{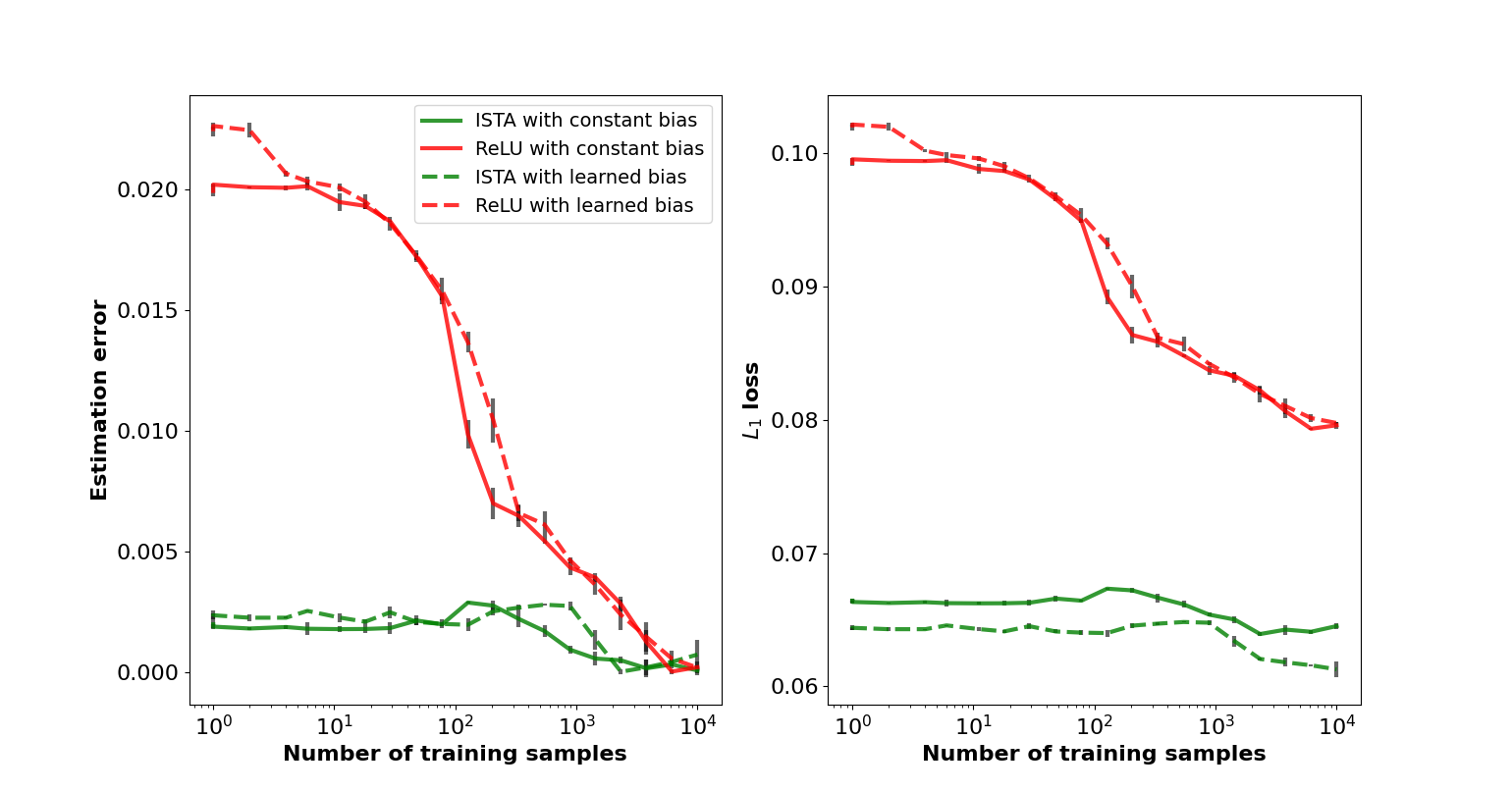}
        \caption{Comparing the EE of ISTA and ReLU networks with $2$ layers.}
    \label{fig:ISTA vs ReLU depth 2}
\end{figure}

\begin{figure}[!h]
        \centering
        \includegraphics[width = 4.3in]{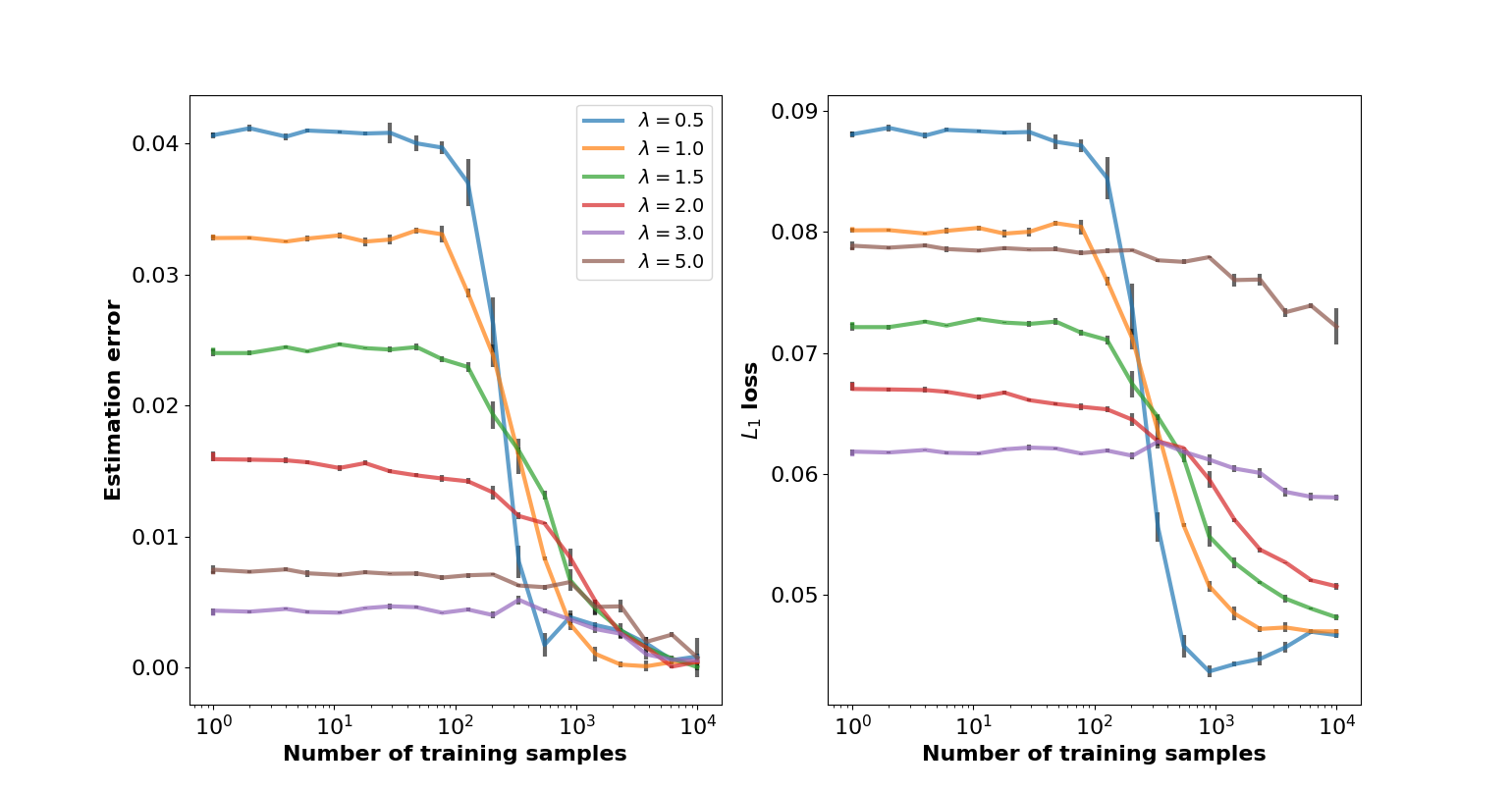}
        \caption{Estimation error and loss of ISTA networks with $2$ layers, as a function of the soft-threshold's value $\lambda$.}
    \label{fig:ISTA lambda depth 2}
\end{figure}

\begin{figure}[!h]
        \centering
        \includegraphics[width = 4.3in]{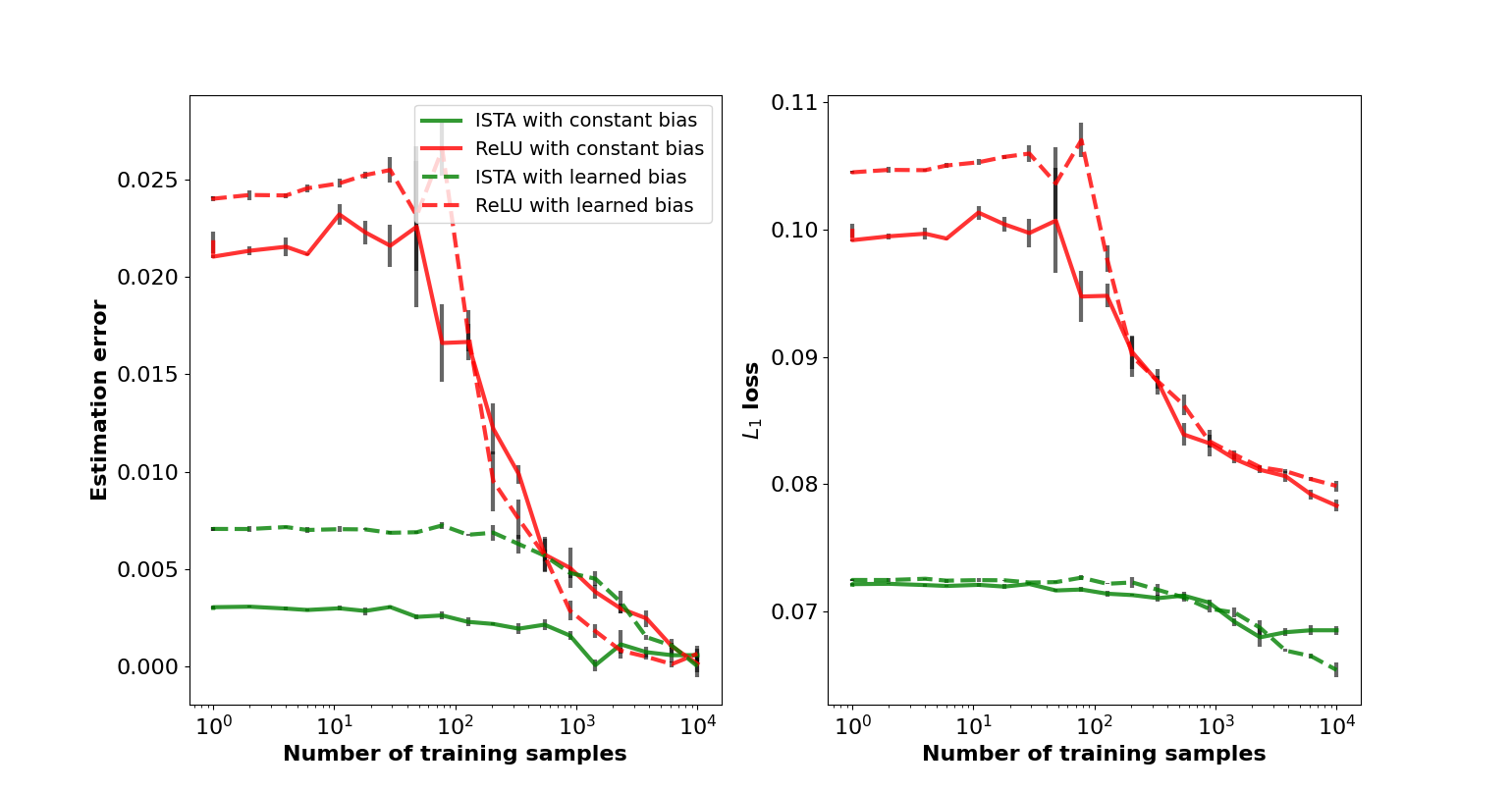}
        \caption{Comparing the EE of ISTA and ReLU networks with $4$ layers.}
    \label{fig:ISTA vs ReLU depth 4}
\end{figure}

\begin{figure}[!h]
        \centering
        \includegraphics[width = 4.3in]{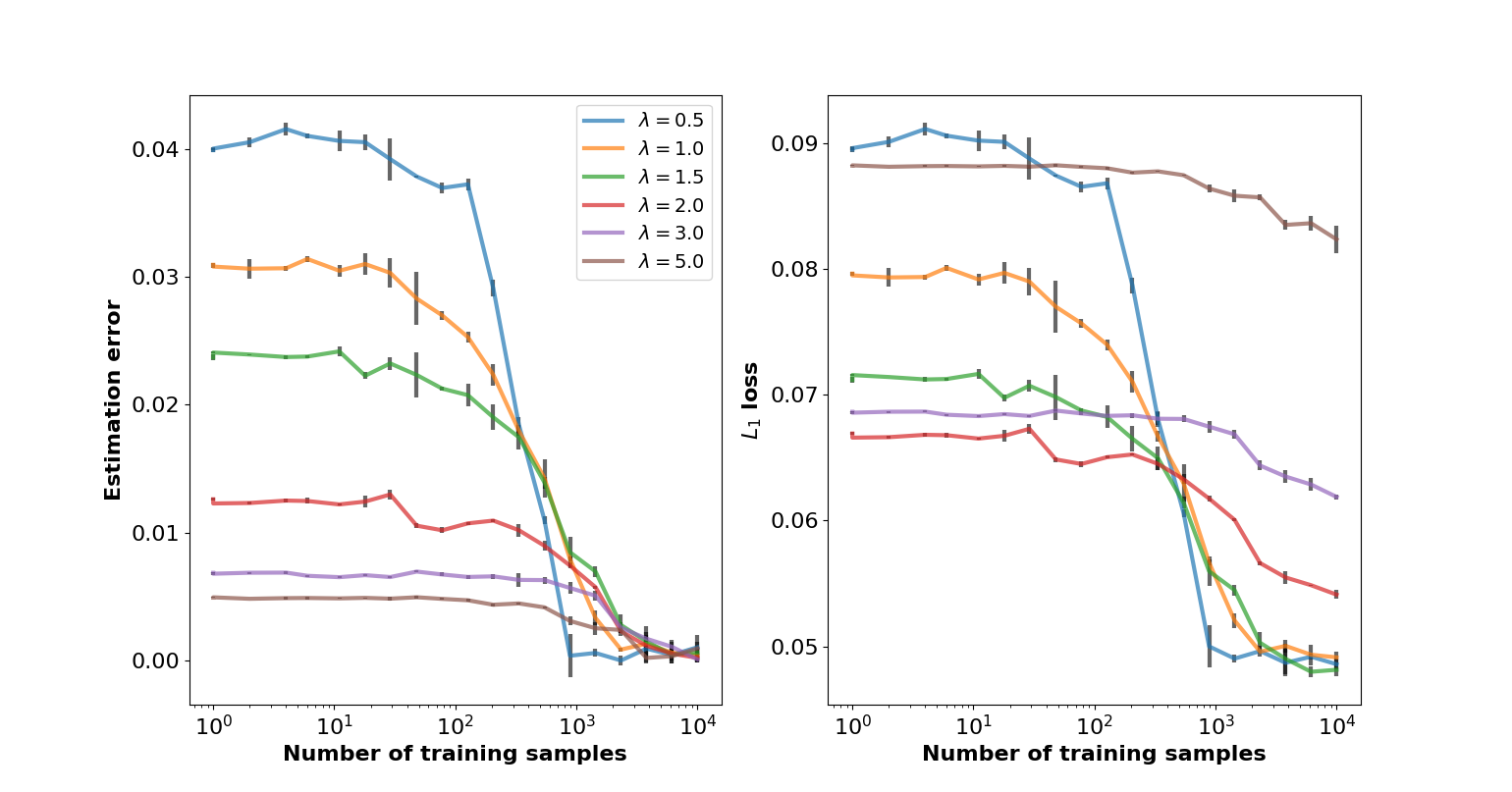}
        \caption{Estimation error and loss of ISTA networks with $4$ layers, as a function of the soft-threshold's value $\lambda$.}
    \label{fig:ISTA lambda depth 4}
\end{figure}

\begin{figure}[!h]
        \centering
        \includegraphics[width = 4.3in]{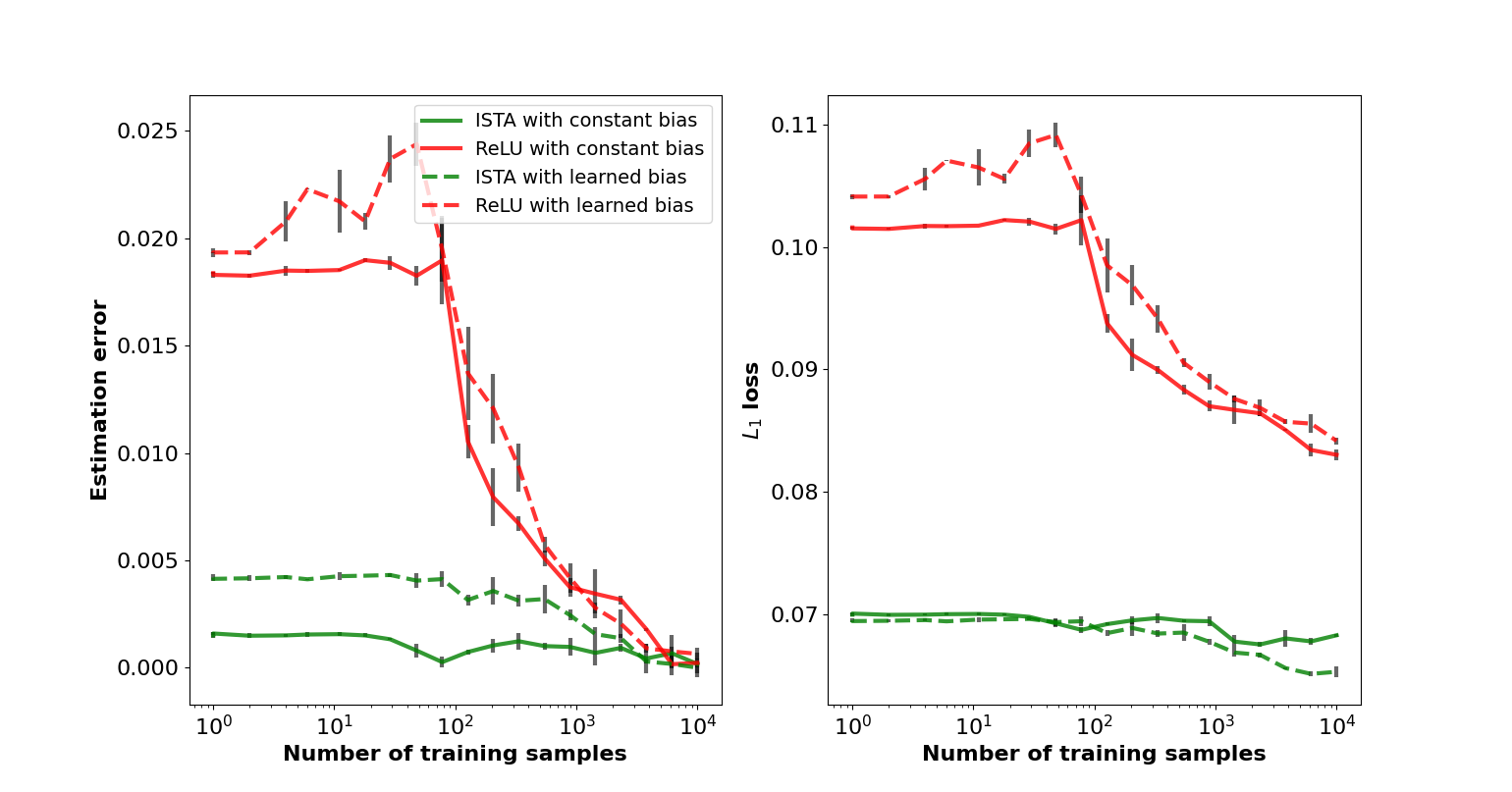}
        \caption{Comparing the EE of ISTA and ReLU networks with $6$ layers.}
    \label{fig:ISTA vs ReLU depth 6}
\end{figure}

\begin{figure}[!h]
        \centering
        \includegraphics[width = 4.3in]{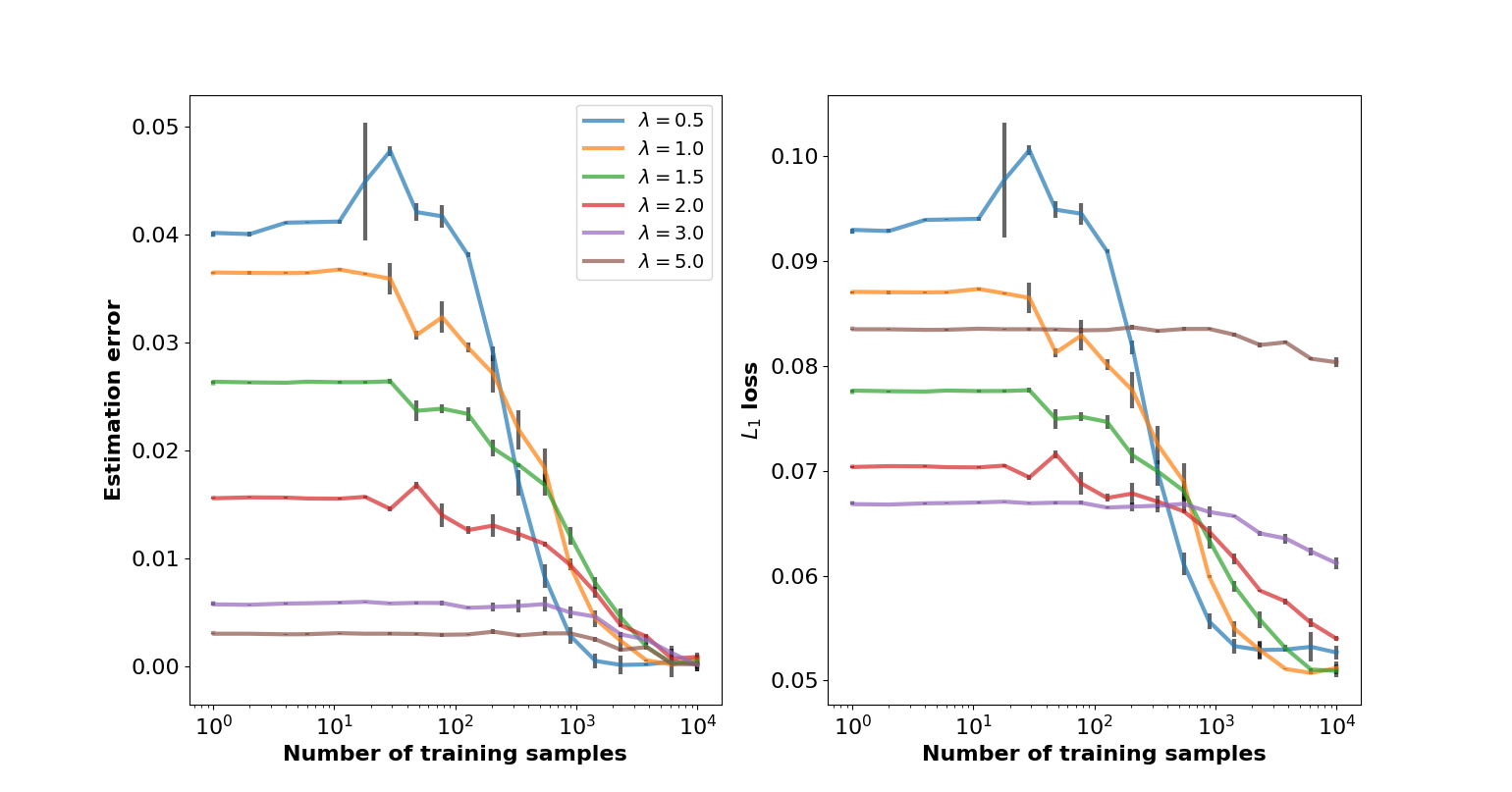}
        \caption{Estimation error and loss of ISTA networks with $6$ layers, as a function of the soft-threshold's value $\lambda$.}
    \label{fig:ISTA lambda depth 6}
\end{figure}

\begin{figure}[!h]
        \centering
        \includegraphics[width = 4.3in]{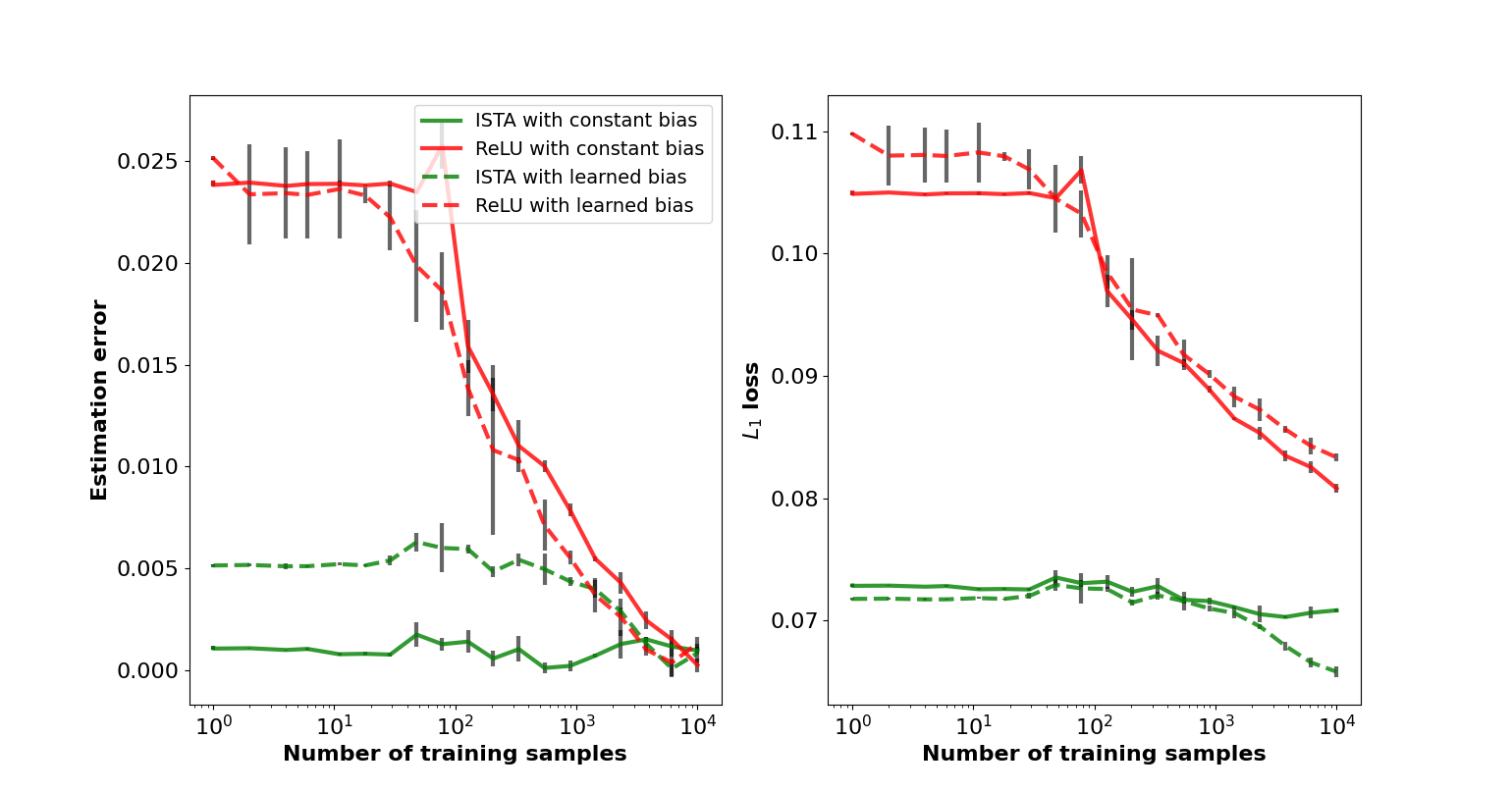}
        \caption{Comparing the EE of ISTA and ReLU networks with $8$ layers.}
    \label{fig:ISTA vs ReLU depth 8}
\end{figure}

\begin{figure}[!h]
        \centering
        \includegraphics[width = 4.3in]{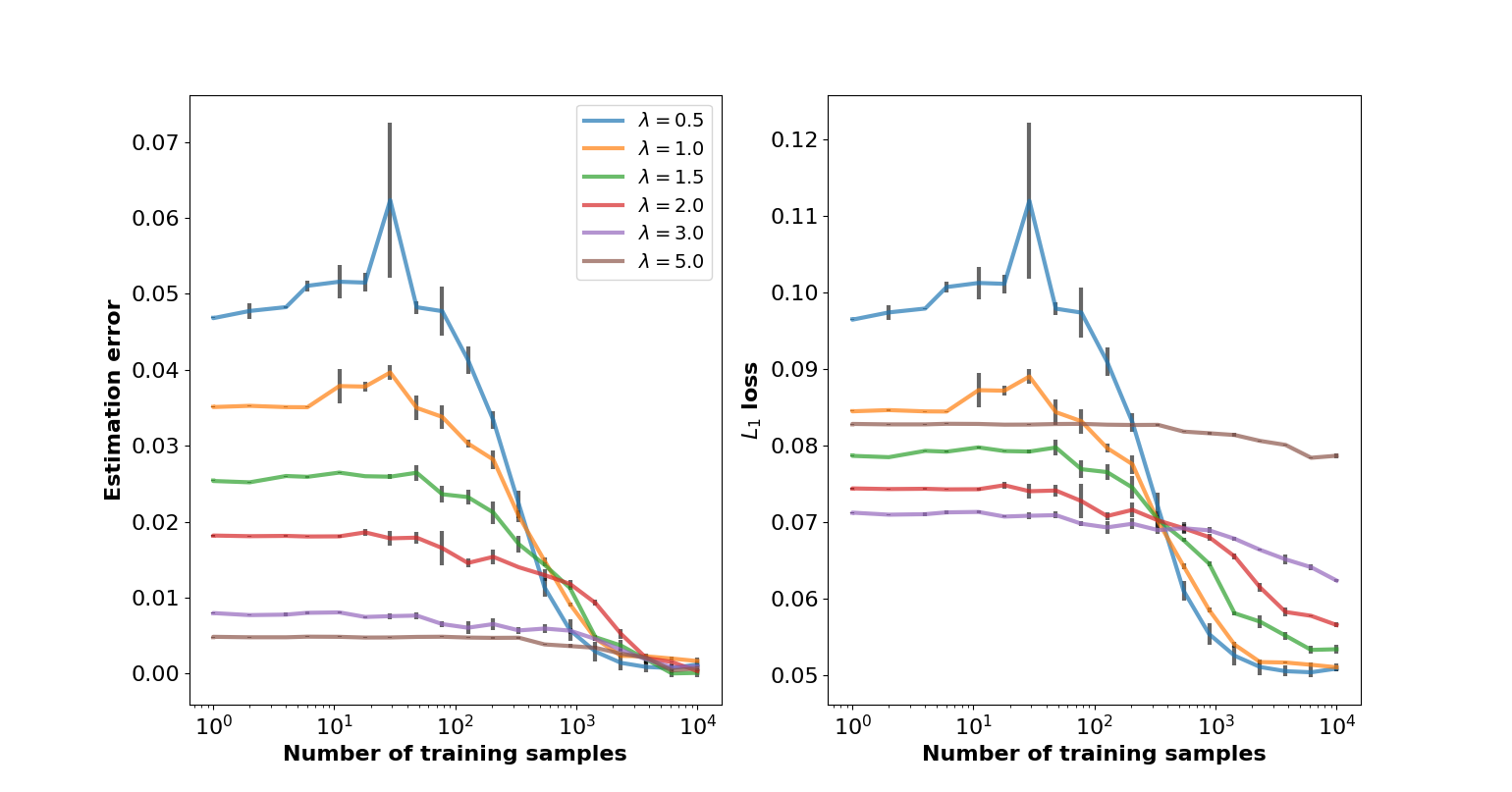}
        \caption{Estimation error and loss of ISTA networks with $8$ layers, as a function of the soft-threshold's value $\lambda$.}
    \label{fig:ISTA lambda depth 8}
\end{figure}

\end{document}